\documentclass{article}

\usepackage[utf8]{inputenc} %
\usepackage[T1]{fontenc}    %
\usepackage{color}
\definecolor{asdf}{rgb}{0.1333,0.5451,0.1333}
\usepackage[colorlinks,linkcolor=blue,citecolor=blue,urlcolor=blue]{hyperref}
\usepackage{url}            %
\usepackage{booktabs}       %
\usepackage{amsfonts}       %
\usepackage{nicefrac}       %
\usepackage{microtype}      %

\usepackage{url}
\usepackage{graphicx}
\usepackage[font=scriptsize]{caption}
\usepackage{amsmath}

\usepackage{amsfonts}
\usepackage{subcaption}
\usepackage{graphicx} %

\usepackage{xcolor}
\usepackage{algorithm,algorithmic}
\usepackage{tabularx} %
\usepackage{amsthm}
\usepackage{amsfonts}
\usepackage{amsmath}
\usepackage{amssymb}
\usepackage{mathrsfs}
\usepackage{multirow}
\usepackage{bbm}
\usepackage[normalem]{ulem}

\renewcommand{\subset}{\subseteq}

\newtheorem{theorem}{Theorem}
\newtheorem{definition}{Definition}

\newtheorem{proposition}{Proposition}
\newtheorem{problem}{Problem}
\newtheorem{lemma}{Lemma}
\newcounter{assump}
\newtheorem{assumption}[assump]{Assumption}

\DeclareMathOperator*{\argmax}{arg\,max}
\DeclareMathOperator*{\argmin}{arg\,min}

\DeclareMathOperator*{\var}{var}

\usepackage{enumitem}
\usepackage{wrapfig}

\newcommand{\beq}{\begin{equation}}

\newcommand{\eeq}{\end{equation}}
\newcommand{\beqs}{\begin{equation*}}
\newcommand{\eeqs}{\end{equation*}}

\newcommand{\poly}{\text{poly}}

\renewcommand{\AA}{\mathcal{A}}

\newcommand{\OO}{\mathcal{O}}

\newcommand{\BB}{\mathcal{B}}

\newcommand{\R}{\mathbb{R}}

\newcommand{\EE}{\mathcal{E}}

\newcommand{\FF}{\mathcal{F}}

\newcommand{\GG}{\mathcal{G}}
\newcommand{\MM}{\mathcal{M}}
\newcommand{\N}{\mathbb{N}}

\newcommand{\regret}{\textbf{Reg}}

\newcommand{\E}{\mathbb{E}}
\newcommand{\CC}{\mathcal{C}}
\newcommand{\XX}{\mathcal{X}}

\newcommand{\Tau}{\mathcal{T}}

\newcommand{\abs}[1]{\ensuremath{| #1 |}}
\newcommand{\floor}[1]{\lfloor #1 \rfloor}
\newcommand{\ceil}[1]{\lceil #1 \rceil}

\newcommand{\train}{n_{\text{train}}}
\newcommand{\valid}{n_{\text{valid}}}
\newcommand{\tol}{\textsc{Tol}}

\newcommand{\balpha}{\pmb{\alpha}}
\newcommand{\bzeta}{\pmb{\zeta}}

\newcommand{\appr}{\textsc{Approx}}
\newcommand{\comp}{\textsc{Comp}}
\newcommand{\alg}{\textsc{ModBE}}

\usepackage{thm-restate}

\renewcommand{\Pr}{\mathbb{P}}
\renewcommand{\P}{\mathcal{P}}

\usepackage[round]{natbib}

    \usepackage[final]{neurips_2022}

\usepackage[utf8]{inputenc} %
\usepackage[T1]{fontenc}    %
\usepackage{hyperref}       %
\usepackage{url}            %
\usepackage{booktabs}       %
\usepackage{amsfonts}       %
\usepackage{nicefrac}       %
\usepackage{microtype}      %
\usepackage{xcolor}         %
\usepackage[subtle]{savetrees}

\title{Oracle Inequalities for Model Selection\\in Offline Reinforcement Learning
}

\author{%
  Jonathan N. Lee \\
  Stanford University\\
  \texttt{jnl@stanford.edu} \\
   \And
   George Tucker \\
   Google Research \\
   \texttt{gjt@google.com} \\
   \AND
   Ofir Nachum \\
   Google Research \\
   \texttt{ofirnachum@google.com} \\
   \And
   Bo Dai \\
   Google Research \\
   \texttt{bodai@google.com} \\
   \And
   Emma Brunskill \\
   Stanford University \\
   \texttt{ebrun@cs.stanford.edu} \\
}

\begin{document}

\maketitle

\begin{abstract}
    In offline reinforcement learning (RL), a learner leverages prior logged data to learn a good policy without interacting with the environment. A major challenge in applying such methods in practice is the lack of both theoretically principled and practical tools for model selection and evaluation. To address this, we study the problem of model selection in offline RL with value function approximation. The learner is given a nested sequence of model classes to minimize squared Bellman error and must select among these to achieve a balance between approximation and estimation error of the classes. We propose the first model selection algorithm for offline RL that achieves minimax rate-optimal oracle inequalities up to logarithmic factors.
    The algorithm, \alg{}, takes as input a collection of candidate model classes and a generic base offline RL algorithm. By successively eliminating model classes using a novel one-sided generalization test, \alg{} returns a policy with regret scaling with the complexity of the \textit{minimally complete} model class. In addition to its theoretical guarantees, it is conceptually simple and computationally efficient, amounting to solving a series of square loss regression problems and then comparing relative square loss between classes. We conclude with several numerical simulations showing it is capable of reliably selecting a good model class.\footnote{Supplementary material is available at: \url{https://sites.google.com/stanford.edu/offline-model-selection}.}
\end{abstract}

	\section{Introduction}

	Model selection is a fundamental task in supervised learning and statistical learning theory. Given a sequence of model classes, the goal is to optimally balance the approximation error (bias) and estimation error (variance) offered by the potential model class choices, even though the best model class is not known in advance. Model selection algorithms are extremely well-studied in learning theory \citep{massart2007concentration, lugosi1999adaptive,bartlett2002model,bartlett2008fast}, and methods like cross-validation have become essential steps for practitioners. 
	
	In recent years, interest has turned to model selection in decision-making problems like bandits and reinforcement learning. A number of theoretical works have studied the \textit{online} setting \citep{agarwal2017corralling,foster2019model,pacchiano2020model,lee2021online,modi2020sample,chatterji2020osom,muthukumar2021universal}. Similar to the bias-variance balance in supervised learning, these algorithms typically aim to select the model class with smallest statistical complexity that contains the true model.
	Despite these recent efforts, the current understanding of model selection in \emph{offline} (or batch) reinforcement learning (RL) is comparatively nascent. Offline RL is a  paradigm where the learner leverages prior datasets of logged interactions with the environment~\citep{lange2012batch,levine2020offline}. The learner is tasked with returning a good policy without further environment interaction.
	As has been acknowledged in several recent papers \citep{xie2021batch,mandlekar2021matters,kumar2021workflow},  one of the major challenges preventing widespread deployment of offline RL algorithms in the real world is the lack of algorithmic tools for model selection, evaluation, and hyperparameter tuning. 
	In experimental settings, researchers typically evaluate candidate learned models by using online rollouts of the policies after learning with offline data. However, such approaches are not feasible in many real world settings where the entire process of producing a single policy must be conducted only on the offline dataset, due to complications such as logistics, safety, or performance requirements. %
	
	In recent years, this problem has been recognized as a major deficiency in the field and a number of efforts have been made to remedy it. On the empirical side, several researchers have proposed workflows and general heuristics specifically addressing this problem \citep{kumar2021workflow,tang2021model,paine2020hyperparameter}. However, all have noted that solutions designed to evaluate or select models typically have their own  hyperparameters and modeling choices. Consider, for example, applying off-the-shelf offline policy evaluation (OPE) methods \citep{precup2000eligibility,thomas2016data}. These typically require some function approximation of their own. Thus, rather than solving the problem, naively using OPE just shifts the burden of model selection to the OPE estimator.  Similarly, recent efforts to solve model selection in \textit{online} bandits and RL are inapplicable as they almost universally require interaction with the environment \citep{foster2019model,pacchiano2020model, lee2021online}.
	The solution to the \textit{offline} problem seems to require new ideas.

	On the theoretical side, there is also significant motivation for devising model selection algorithms as there is growing evidence suggesting that strong conditions on the function class\footnote{That is, conditions sufficient for supervised learning, like realizability, tend not to be sufficient on their own for offline RL.} are necessary to achieve non-trivial guarantees in offline RL in the worst case \citep{foster2021offline,zanette2021exponential,wang2020statistical}. Perhaps the most widely used and recognized condition is completeness \citep{munos2008finite,antos2008learning,chen2019information} which essentially says that $\Tau f \in \FF$ for any $f \in \FF$, where $\Tau$ is the Bellman operator and $\FF$ is the model class.\footnote{$\FF$ is a model class meant to estimate $Q$-functions. It consists of functions mapping state-action pairs to value predictions. The Bellman operator applied to $f \in \FF$ pointwise is defined as $\Tau f(x, a) = r(x, a) + \max_{a'} \E_{x'  | x, a} f(x', a')$.} Unsurprisingly, completeness plays an important role in the proofs of many value-based offline RL algorithms since sample efficient results are provably impossible without it (in the absence of additional assumptions -- see \cite{xie2021batch,zhang2021towards}). Despite the growing realization of the importance of these conditions, there seems to be comparatively little work addressing the problems of identifying complete model classes or certifying sufficient conditions for sample efficient offline RL. 
	
		\cite{lee2021model} considered the problem of model selection in the offline setting with the intent of addressing some of the aforementioned issues. It was shown that full model selection (competitive with an oracle that has knowledge of the best model class) is  impossible in general in offline reinforcement learning. They proposed several relaxations to achieve weaker oracle inequalities, but these were limited to contextual bandits with linear model classes where there is no issue of completeness. The question of whether any similar results are possible for full offline reinforcement learning with general function classes has remained open.

	\subsection{Contributions}
	
	\paragraph{Theoretical Guarantees} In this paper, we give the first rate-optimal model selection algorithm for offline RL with value function approximation. We begin by summarizing known results for a single model class using value-based methods.
	For any individual model class $\FF$ that satisfies completeness and an offline dataset of $n$ samples with sufficient coverage, the gold-standard regret bound is $\tilde \OO\left( \sqrt{  \comp(\FF) /n }    \right) $\footnote{For clarity, $\tilde \OO$ omits dependence on certain parameters such as the horizon $H$, distribution mismatch factors, number of classes $M$, failure probability $\delta$, log factors, and constants.} where $\comp(\FF)$ denotes the statistical complexity of $\FF$. This is achieved, for example, by Fitted Q-Iteration (FQI) \citep{chen2019information}. Clearly, one would like $\comp(\FF)$ be as small as possible to achieve a tighter bound.
	
	We consider the model selection problem where we are given an offline dataset of $n$ samples and a nested sequence of $M$ model classes $\FF_1 \subset \ldots \subset \FF_M$. We investigate the following question: \textit{Can we achieve a model selection guarantee for offline RL with regret scaling with the complexity of the smallest complete model class?  } 
	
	We present a novel  and conceptually simple algorithm, \alg{}, that achieves regret scaling with the complexity of the smallest class satisfying completeness without knowledge of this class \textit{a priori}.
	\begin{theorem}\label{thm::main-informal}(informal version of Corollary~\ref{cor::fqi})
	Given an offline dataset of $n$ samples and nested model classes $\FF_1 \subset \ldots \subset \FF_M$, \alg{} outputs $\hat \pi$ such that $\regret(\hat \pi) =  \tilde \OO\left(\sqrt{ \comp(\FF_{k_*}) / n } \right)$ where $k_* = \min \{ k \in [M] \ : \ \FF_k \text{ is complete } \}$ .
	\end{theorem}

	A guarantee of this nature is typically known as an \textit{oracle inequality} since an oracle with knowledge of the "best" model class ahead of time could simply choose it. We remark that this oracle inequality is rate-optimal in $\comp(\FF_{k_*})$ and $n$, showing that we do not have to sacrifice efficiency for adaptivity. This is in contrast to some other works in model selection for decision-making where this unfortunate efficiency-adaptivity trade-off has been observed \citep{foster2019model,pacchiano2020model,xie2021batch}. In Appendix~\ref{app::hardness}, we discuss how the nestedness condition is necessary.
	
 We also provide a robustness result for model selection (Theorem~\ref{thm::robust}): if no models are Bellman complete (that is, $k_*$ does not exist), \alg{} obtains  %
 $\regret(\hat \pi) \leq\tilde{\OO}\left( \min_{k \in [M]}   \sqrt{\xi_k + \comp(\FF_k) /n } \right)$ where $\xi_k$ is a measure of the \textit{global} completeness error of $\FF_k$.\footnote{See Section~\ref{sec::problem} for a precise definition.} Our results show that, while some model selection problems remain elusive without further assumptions, strong rate-optimal oracle inequalities are still possible under standard offline RL assumptions even without knowledge of the best classes in advance.

	\paragraph{Technical Highlights.} The key to achieving the near optimal regret rate is to achieve the near optimal excess risk rate of the squared Bellman error (which is of order $\tilde{\OO}\left({ \comp(\FF_{k_*} ) / n } \right)$). To do this, \alg{} iteratively compares the relative effectiveness of two candidate model classes  by employing a hypothesis test that compares the difference of their estimated risks to a \textit{one-sided} generalization bound. The fact that the test leverages only the one-sided generalization bound is crucial: %
	using easier two-sided bounds (e.g. from uniform deviation bounds on risk estimators) leads to a squared Bellman error rate of $\tilde \OO\left( \sqrt{ \comp( \FF_{k_*}) / n } \right) $, which translates to a slow $\tilde \OO((\comp( \FF_{k_*}) / n)^{1/4})$ regret rate. %
	Instead the 
	one-sided generalization error 
	allows us to ultimately obtain the optimal $\tilde \OO\left( \sqrt{ \comp( \FF_{k_*}) / n } \right) $ regret rate. 
	
	\paragraph{Practical Results.} In practice, \alg{} can be instantiated with \emph{any} base offline RL algorithm that attempts to minimize squared Bellman error, including but not limited to FQI. \alg{} is also computationally efficient, requiring $\OO(H k_* M)$ calls to an empirical squared loss minimization oracle and $\OO(k_*)$ calls to the base offline RL algorithm. In Section~\ref{sec:experiments}, we demonstrate the effectiveness of \alg{} on several simulated experimental domains. We use neural network-based offline RL algorithms as baselines and show that  \alg{} is able to reliably select a good model class.

	\subsection{Additional Closely Related Work}

	Several prior works have specifically set out to address the model selection problem from a theoretical perspective, as we do here. \cite{lee2021model} formalized the end-to-end model selection problem for offline RL where, given nested model classes, the goal is to produce a regret bound competitive with an oracle that has knowledge of the optimal model class. 
	Their positive results, however, were limited only to linear model classes for contextual bandits; ours apply to sequential settings.  An earlier work by \cite{farahmand2011model} had partially addressed our problem but made several restrictive assumptions such as a known generalization bound that underestimates the approximation error (which is generally unknown); our algorithm only relies on commonly known quantities. Another notable work is the BVFT algorithm of \cite{xie2021batch}. While initially designed for general policy optimization, BVFT can be applied to model selection \citep{zhang2021towards} but it incurs a slow $1/n^{1/4}$ regret rate in theory (compared to our $1/n^{1/2}$) and requires a stronger data coverage assumption. One advantage of BVFT is that it can be used more generally to tune hyperparameters beyond the selection of model classes. However, the specialization of our algorithm to model selection enables the stronger guarantees. Thus,  we view the two algorithms as complementary.  \cite{jiang2015abstraction} studied abstraction selection between nested state abstractions of increasing granularity; however, this eschews problems specific to value function approximation setting. \cite{hallak2013model} studied a similar abstraction problem, giving only asymptotic guarantees.
	In Section~\ref{sec::limitations}, we will discuss in more detail why several seemingly natural approaches to model selection do not produce satisfactory results.

	\section{Preliminaries}
	
	\paragraph{Notation} 
	For any $n \in \N$, we let $[n] = \{1, \ldots, n\}$. The notation $a \lesssim b$ implies that $a \leq C b$ for some absolute constant $C > 0$. We will use $C, C_1, C_2 \ldots > 0$ to denote absolute constants (independent of problem parameters). For a set $A$, $\Delta(A)$ denotes the set of distributions over $A$.

	We consider the finite-horizon Markov decision process $\MM (\XX, \AA, H, \Pr, r, \rho)$ where $\XX$ is the (potentially infinite) state-space, $\AA$ is the action space, $H$ is the length of the horizon, $\Pr: \XX\times \AA \to \Delta(\XX)$ is the transition kernel, $r: \XX \times \AA \to [0, 1]$ is a deterministic reward function, and $\rho \in \Delta(\XX)$ is an initial state distribution. A learner interacts with the MDP by proposing an $H$-step policy $\pi = ( \pi_{h})_{h \in [H]}$ where each $\pi_h: x \mapsto \pi_h(\cdot | x)$ maps $x \in \XX$ to a distribution over actions in $\Delta(\AA)$.\footnote{With some abuse of notation, for deterministic $\pi_h$ we write $a = \pi_h(x)$ to denote its highest-probability action.} At step $h = 1$, $x_1$ is drawn according to $\rho$. Then at step $h \in [H]$, the agent observes $x_h$, draws $a_h$ according to $\pi_h(\cdot | x_h)$ observes reward $r(x_h, a_h)$ and the MDP transitions to $x_{h + 1}$ according to $\Pr(\cdot | x_h, a_h)$. For a policy $\pi$, we let $P^\pi_h(x, a)$ and $P^\pi_h(x)$ denote the marginal state-action and state densities  of $\pi$ respectively at step $h$.
	
	Following standard definitions, we let $V_h^\pi: \XX \to \R$ denote the value function of $\pi$ at step $h \in [H]$ which is given by $
	V_h^\pi(x) = \E_\pi \left[  \sum_{s \geq h}  r(x_s, a_s)   \ \vline\ x_s = x  \right].$
	Here, the expectation $\E_\pi$ is over trajectories under $\pi$ with $a_h \sim \pi_h(\cdot | x_h)$. Similarly, the action-value function $Q^\pi_h:\XX\times \AA \to\R$ is defined as
	$ %
	Q_h^\pi(x, a) = \E_\pi \left[  \sum_{s \geq h}  r(x_s, a_s)   \ \vline\ x_s = x, a_s = a  \right].
	$%
	The optimal policy (which exists under mild conditions when $H$ is finite~\citep{sutton2018reinforcement}) is denoted by $\pi^*$ and this maximizes $V^{\pi}_h(x)$ for all $x$ and $h$. The average value of a policy $\pi$ is given by $v(\pi):= \E_{x \sim \rho} \left[V_1^\pi(x) \right]$. Finally, we define the Bellman operators:
	$T^\pi_h Q (x, a) = r(x, a) + \E_{x' \sim P(\cdot | x, a), a' \sim \pi_{h + 1}(\cdot | x')  }  \left[ Q(x',a')  \right]$ and %
	$T^*_h Q (x, a)  = r(x, a) + \E_{x' \sim P(\cdot | x, a) }  \left[\max_{a' \in \AA} Q(x',a')  \right].$
	Note that the values of $v(\pi)$,  $V^\pi_h$, and $Q^\pi_{h}$ are always in $[0, H]$ due to the constraint on $r$. For convenience, we denote the $Q$ function of the optimal policy as $Q^* = Q^{\pi^*}$. %

	We consider the setting where the learner is provided with a model class $\FF \subset (\XX \times \AA \to [0, H])$ to estimate action value functions at each step. For exposition, we assume this model class is \emph{finite}; however, it is straightforward to extend to infinite settings with appropriate complexity measures. For simplicity, we will assume that the learner uses the same $\FF$ for each timestep $h \in [H]$ but this is trivially extended. We assume that $0 \in \FF$ and we  always write $f_{H + 1} = 0$. For any function  $f \in \XX\times \AA \to [0, H]$, we define the argmax policy
	$\pi_{f} (x) = \argmax_{a \in \AA} f(x, a)$.
	We will also write $f(x) = \max_{a \in \AA} f(x, a)$.

	\subsection{Offline Reinforcement Learning}
	
	The distinguishing feature of the offline (or batch) RL is that we assume that the learner is provided with a dataset $D$ of example transitions in the MDP. The learner itself is not permitted to interact in the environment. The objective is to produce a good policy $\hat \pi$ using only data from the dataset $D$.
	
	Formally, the dataset decomposes as $D = (D_h)_{h \in [H]}$ for each timestep where $D_h = \{ ( x, a, r, x')\}$ consists of tuples of transitions and incurred rewards. We assume $D_h$ contains $n$ datapoints that are sampled i.i.d from a fixed marginal distribution $\mu_h \in \Delta(\XX \times \AA)$ and the data are independent across timesteps $h$. That is, there are $Hn$ datapoints total. For example, the data could be generated from $h$-step state-action distribution of a behavior policy $\pi^b$ so that $\mu_h(x,a) = P_h^{\pi^b}(x, a) =   \pi^b_h(a | x ) P_h^{\pi_b}(x)$.
	
	For $f, g \in (\XX \times \AA \to \R)$, we use the notation $\| f- g \|_{\mu_h}^2 = \E_{\mu_h}\left[ ( f(x, a) - g(x, a) )^2 \right]$. The average squared Bellman error under $\mu$ at state $h$ with respect to $f, g$ is $\| f- T^*_h g \|_{\mu_h}^2$.
	Following classical conventions \citep{munos2008finite,duan2021risk}, we make a concentrability assumption that the data distribution $\mu$ has good coverage over the MDP for all reachable state-actions.
	\begin{assumption}
		There exists a constant $\CC(\mu) > 0$ such that
		$
		\sup_{h, x, a, \pi}  { P^\pi_h(x, a) \over \mu_h(x, a)} \leq \CC(\mu)
		$.
	\end{assumption}
	
	Concentrability is a structural assumption and it is widely regarded as perhaps the most standard assumption when studying offline RL problems \citep{foster2021offline}. We remark that recent theoretical works have striven to weaken this condition via pessimistic methods  \citep{ liu2020provably,jin2021pessimism,xie2021bellman,uehara2021pessimistic}.
	However, Theorem 2 of \cite{lee2021model} shows that model selection bounds of this type are not possible even in contextual bandits and even though the single model class bounds are possible. As a result, we will not consider this refinement in the present paper.
	
	In this offline setting, the learner aims to use $D$ and $\FF$ to produce a policy $\hat \pi$ so as to minimize the regret, which measures the difference in average value between the optimal policy $\pi^*$ and $\hat \pi$:
	\begin{align}
	\regret(\hat \pi) := v(\pi^*) - v(\hat \pi).
	\end{align}
    The following variant of the performance difference lemma will be used throughout the paper. It shows that it is sufficient to control the squared Bellman error to bound regret.
	
	\begin{lemma}[\cite{duan2021risk}]\label{lem::perf-diff}
	    For any $f_1, \ldots, f_H$, let $\pi := ( \pi_{f_h})_{h \in [H]}$. Then, 
	       $
	       \regret(\pi)  \leq 2 \sqrt{ \CC(\mu) \sum_{h \in [H]} \| f_h - T^*_h f_{h + 1} \|^2_{\mu_h}}.
	       $
	\end{lemma}

	\section{Model Selection Objectives}
	
	In this section, we state our primary model selection objectives and discuss their significance as well as challenges associated with solving them.

	\subsection{The Model Selection Problem}
	\label{sec::problem}
	
		For a finite function class $\FF$ that we consider here, the gold-standard regret guarantee for offline algorithms with value function approximation is \begin{align}\label{eq::single}
		\textstyle{
		\regret(\hat \pi) = \tilde{\OO}\left(  \sqrt{\CC(\mu)\appr(\FF) } + \sqrt{ \CC(\mu)\log|\FF| \over n } \right),}
		\end{align}
		 where $\appr(\FF) := \max_{h \in [H], f' \in \FF} \min_{f \in \FF}  \| f - T^*_h f' \|_{\mu}^2$ is the completeness error of the class $\FF$ \citep{chen2019information}. This is achieved, for example, by the Fitted Q-Iteration (FQI) algorithm. If we were using infinite classes, we would replace $\log |\FF|$ with another suitable notion of complexity such as pseudodimension. Such bounds naturally exhibit a trade-off: larger function classes may have a better chance of keeping  $\appr(\FF)$ close to zero\footnote{In contrast to realizability, this intuition of monotonicity of $\appr(\FF)$ is not universally true for completeness. Adding functions to the class $\FF$  might actually \textit{increase} $\appr(\FF)$. However, it remains a useful heuristic. In Appendix~\ref{app::hardness}, we discuss how model selection in this setting is not possible without  nestedness.} but require more data to minimize the estimation error. Small classes face the opposite problem.  

\begin{definition}
A class $\FF$ is complete if 
$\appr(\FF) :=  \max_{h \in [H], f' \in \FF} \min_{f \in \FF}  \| f - T^*_h f' \|_{\mu}^2 = 0.$
\end{definition}

The objective of model selection is to achieve refined regret bounds that balance approximation error and estimation error. To this end, we assume that the learner is presented with not just a single model class $\FF$, but rather a nested sequence of $M$ classes $\FF_1 \subset \ldots \subset \FF_M$. Solving a problem with nested model classes is  common practice in both supervised learning and offline RL. For example, one often starts with an extremely large class $\FF$ and then considers restrictions of $\FF$ to an increasing sequence $\FF_1 \subset \ldots \subset \FF_M = \FF$. In a linear setting, this could correspond to trying to find a subset of candidate features that are sufficient to solve the problem.

Since the approximation error is typically unknown \textit{a priori}, we aim to design an algorithm capable of selecting a good class in a data-dependent manner. In particular, we would like to achieve \textit{oracle inequalities} reflecting that we can compete with the performance of an oracle that has this knowledge in advance. 

Our primary objective is to compete with the \textit{minimally complete} model class. 

\begin{problem}\label{prob::min-complete}
Let $k_* = \min \{ k \in [M]   : \ \FF_k \text{ is complete}\}$. Find $\hat \pi$ with $\regret(\hat \pi)=\tilde{\OO}( \sqrt{ \CC(\mu)\log(|\FF_{k_*}|)/ n })$.
\end{problem}

Here, $\FF_{k_*}$ is the smallest class that satisfies completeness on the data distribution. Such oracle inequalities are common in model selection for online bandits and RL \citep{foster2019model} -- albeit they are generally not rate-optimal in that literature. In particular, Problem~\ref{prob::min-complete} states the regret bound should %
achieve the same dependence on $\log |\FF_{k_*}|$ and $n$, as would an optimal offline algorithm using a single class with $k=k_*$.
In other words, we do not tolerate any worse dependence on either quantity such as $\tilde \OO(1/n^{1/4})$ rates and other lower order terms.

We are also interested in a \textit{robustness}  when $k_*$ may not exist, i.e. all $\FF_k$ have some approximation error. 

\begin{problem}\label{prob::robust}
Define the global completeness error as $\xi_k := \max_{h \in [H], f' \in \FF_M} \min_{f \in \FF_k} \| f - T_h^* f' \|_{\mu_h}^2$. Find $\hat \pi$ so that  $\regret(\hat \pi) =   \tilde{\OO}\left(  \min_{k \in [M]} \left\{ \sqrt{\CC(\mu) \xi_k} + \sqrt{ \CC(\mu)\log(|\FF_k|)/ n } \right\}\right)$
\end{problem}

Note that $\xi_k \geq \appr(\FF_k)$ by definition. For the estimation error, however, the guarantee remains rate-optimal. We remark that a solution to one of the above problems does not immediately imply a solution to the other. For example, a class $\FF_k$ may be complete, but $\xi_k$ can still be large. Perhaps surprisingly, our proposed algorithm will be able to handle both problems \emph{simultaneously} without knowledge of whether $k_*$ exists, thus achieving the $\min$ of both oracle inequalities.

	\subsection{Limitations of Prior Approaches}
	\label{sec::limitations}
	
We now  review some of the core challenges involved in solving the above problems. There are a number of seemingly natural approaches to model selection in RL that are surprisingly unable to produce satisfactory results, at least off-the-shelf.

	\paragraph{Adaptive offline policy evaluation} The most natural approach, to which we have alluded in the introduction, is to
	first compute $\hat \pi_k$ with a base algorithm using function class $\FF_k$, for each $k \in [M]$. Then, one can estimate $v(\hat \pi_k)$ using an off-the-shelf offline policy evaluation approach such as fitted $Q$-evaluation \citep{munos2008finite,duan2020minimax}, DICE methods \citep{nachum2019dualdice,dai2020coindice,zhan2022offline}, marginalized importance estimators \citep{xie2019towards}, or doubly robust estimators \citep{jiang2016doubly,thomas2016data}. Then one simply picks the $\hat \pi_k$ with the best estimated value. The main drawback of this approach is that nearly all of the above methods require selecting a model class to perform the estimation,\footnote{In the case of marginalized importance sampling, the guarantee is not strong enough to compete with the oracle.} and it is unclear how to balance the estimation and approximation error optimally to compete with the oracle. 
	One possible solution is to employ the adaptive estimator of \cite{su2020adaptive}, which takes as inputs a sequence of offline estimators and known upper bounds on their deviations and returns an estimator that competes with the best one.  
	This is precisely the approach taken by \cite{lee2021model} for linear contextual bandits. 
	However, for general function classes in RL, there is no obvious way to compute the analogous deviation bounds, which oftentimes depend on the unknown quantity $\CC(\mu)$. Since these bounds are required by the adaptive estimator as inputs, we are yet again left with unknown hyperparameters to tune.

	\paragraph{Bellman error estimators} 
	Recall we are focusing on base offline RL algorithms that attempt to minimize the squared Bellman error of objective. Therefore, one might ask whether it is possible to estimate the Bellman errors (e.g. with the validation dataset) and compare the model classes using the Bellman error as a proxy. Consider, for example, FQI which iteratively minimizes the squared Bellman error:
	\begin{align*}
	    \hat f_{h} = \argmin_{f \in \FF_k} \hat \E_{D_h} \left[ \left(  f(x, a) - r - \max_{a'} \hat f_{h + 1}(x', a')\right)^2  \right],
	\end{align*}
	where we use $\hat \E_{D_h}$ to denote the empirical mean calculated with samples from the dataset $D_h$. Presumably, we could simply choose the model class $\FF_k$ that has the smallest cumulative squared error.
	The main issue with this approach is the classic double-sampling problem \citep{baird1995residual,duan2021risk}:  the standard estimator of the Bellman error is biased, as a result of using an empirical version of the Bellman operator $T^*$. By selecting based on this error function alone, we will end up favoring model classes that also induce low variance of the \textit{regression targets}, given by $r +  \hat f_{h + 1}(x')$ at step $h$. This is because the expectation is given by:
	\begin{align*}
	 \E_{\mu_h} \left[ \left(  \hat f_h(x, a) - r - \hat f_{h + 1}(x')\right)^2  \right] & = \| \hat f_h - T^* \hat f_{h + 1}  \|_{\mu_h}^2 + \E_{\mu_h} \left[  \var_{x' \sim \P(\cdot | x, a)} \left(  \hat f_{h + 1}(x') \right) \right].
	\end{align*}
	In reality, we want to choose a class $\FF_k$ to minimize only the first term on the right-hand side, summed over $h\in[H]$, following Lemma~\ref{lem::perf-diff}. However, the second term is generally unknown. One could  assume there is a sufficiently powerful class $\GG$ such that $T^*f  \in \GG$ for all $f \in \FF$ \citep{chang2022learning}. But there remains a question of how to select the class $\GG$ to trade off approximation error and estimation error, creating another unsolved model selection problem.
	
	In the same vein, another approach we might consider is recent BVFT algorithm of \cite{xie2021batch} to select among the $f^k$ learned by the base algorithm. This  solves the model selection problem but the guarantee of BVFT has a slow $O(1/n^{1/4})$ dependence and thus does not achieve either oracle inequality.
	It also, in theory, requires that a discretization parameter is set based on a concentrability coefficient stronger than $\CC(\mu)$, which is typically unknown. Follow up work has shown this can be chosen adaptively in practice \citep{zhang2021towards}.
	
	Perhaps most conceptually related to our approach is past work which compares Bellman errors of finer-grained state abstraction functions on the Q-function computed on  coarser-grain state abstraction~\citep{jiang2015abstraction}. This work provided bounds on the resulting policy performance of the selected abstraction in  discrete state and action setting, where models are varying levels of state abstractions. However, this work and analysis critically depends on the discrete state and action setting: our work shows how a similar idea can be used in the value function approximation setting, with substantially different tools and analysis techniques. %

	\textbf{Representation Learning}
		Readers familiar with work in representation learning for RL \citep{agarwal2020flambe} might observe that the problem vaguely resembles objectives for selecting feature representations for low rank MDPs such as \cite{modi2021model}. Unfortunately, the problem settings are quite different, and we cannot simply adapt such representation learning algorithms to the model selection problem since they are either insensitive to the model class complexities or they require stronger realizability assumptions. It would be interesting future work to better understand the relationship between these two problems.

\section{\alg{} Algorithm}\label{sec::algorithm}

Having introduced the model selection objectives, we now present our main result, a novel model selection algorithm for offline RL that provably achieves the aforementioned oracle inequalities.  We first give an intuitive sketch of the approach and present the full algorithm in subsequent subsection.
As a thought experiment, we will consider the case when $M = 2$ and a minimally complete class $\FF_{k_*}$ exists.\footnote{While the algorithm requires minimal changes to extend beyond these constraints, there  are some notable analytic challenges in the proof. For general $M$, we cannot guarantee the class returned will be the correct one always -- it may be substantially smaller but with controllable approximation error. When $k_*$ does not exist, there is a chance to "skip" the best model class, so we must show that this is tolerable. }  We will also ignore logarithmic factors and $H$ dependence for now. 
A key algorithmic idea is that we will first start optimistically by guessing that $k_* = 1$. Running a base algorithm like FQI with $\FF_1$ on training data returns the functions $f_1, \ldots, f_H$, which, with high probability, satisfy
\begin{align*}
\textstyle{
    \sum_h  \| f_h - T^*_h f_{h + 1} \|_\mu^2 = \tilde \OO \left( {\CC(\mu)\log (|\FF_1|) \over n }\right) }
\end{align*}
if $k_*$ actually equals $1$. Given these functions, we can pose a square loss regression problem where the regression targets (i.e., the "y's" of the regression problem) are given by the empirical Bellman updates using training data:
\begin{align*}
\textstyle{
    \hat L_h(g, f_{h + 1})  = {1 \over n} \sum_{(x, a, r, x') \in D_h} \left( g(x, a) - r - f_{h + 1}(x')\right)^2.} 
\end{align*}
Let $L_h(f, g) := \E_{\mu_h} \left[\hat L_h(f, g) \right]$. Solving this regression problem for each $h$ over the class $\FF_2$ will generate $g_1, \ldots, g_H \subset \FF_2$. The key insight is that the sequences $(f_h)_h$ and $(g_h)_h$ are both trying to minimize the same empirical square loss function with the same regression targets: $r + f_{h + 1}(x')$. Unlike the Bellman error estimators from the previous section that incur biases, the losses $L_h(f_h, f_{h + 1})$ and $L_h(g_h, f_{h + 1})$ are comparable and estimable from a validation set. By nestedness of $\FF_1 \subset \FF_2$, $\FF_2$ cannot have more approximation error on this regression problem. Provided we can get a good estimate of generalization errors $L(f_h, f_{h + 1})$ and $L(g_h, f_{h + 1})$ with validation data, this naturally brings forth the following \textit{generalization test}:
if 
\begin{align}\label{eq::gen-test}
\textstyle{
L_h(g_h, f_{h + 1} ) < L(f_h, f_{h + 1}) - \tilde \OO \left( {\log (|\FF_1 |) \over n } \right) }\end{align}
reject $\FF_1$ and pick $\FF_2$. Otherwise pick $\FF_1$. 	That is, a switch will occur not when $\FF_{2}$ performs only marginally better than $\FF_1$, but when it performs \textit{substantially} better as measured by the generalization error that we see for both $f_h$ and $g_h$ on this regression problem. If \eqref{eq::gen-test} holds, then there is reason to believe that $\FF_1$ is not complete, making $\FF_2$ the right choice. Crucially, the test only checks for generalization error, so the tolerance term on the right side goes as $\tilde \OO\left( { \log (| \FF_1| ) \over n } \right)$, which is the correct rate for this problem. Thus, if the test turns out to be wrong, we will only lose additive factors of the correct rate.

\begin{figure}
	\begin{algorithm}[H]
		\caption{ Model Selection via Bellman Error (\alg{}) }\label{alg::theory}
		\begin{algorithmic}[1]
			
			\STATE \textbf{Input}: Offline dataset $D = (D_h)$ of $n$ samples for each $h \in [H]$, Base algorithm $\BB$, function classes $\FF_1\subset \ldots \subset \FF_M$, failure probability $\delta \leq 1/e$, and estimation error function $\omega$ for $\BB$. 
			
			\STATE Let $n_{\text{train}} =\ceil{ 0.8 \cdot n}$ and $n_{\text{valid}} = \floor{0.2 \cdot n}$ and split the dataset $D$ randomly into $D_{\text{train}} = ( D_{\text{train}, h})$ of $n_{\text{train}}$ samples and $D_{\text{valid}} = (D_{\text{valid}, h} ) $ of $n_{\text{valid}}$ samples for each $h \in [H]$.
			
			\STATE Set $\bzeta := {  96 H^2 \log (16M^2H /\delta ) \over \valid } $

			\STATE Initialize $k  \leftarrow 1$.

			\WHILE{$k < M$} \label{line::while}
			
				\STATE $ (f_h)_{h \in [H]} \leftarrow \BB(D_{\text{train}}, \FF_k, \delta/4M)$

				\FOR{$k' \leftarrow k + 1, \ldots, M$}

				\STATE Set $\balpha :=  \max \left\{ \omega_{\train, \delta/4M}(\FF_{k'}), { 200 H^2 \log(8 M^2 H |\FF_{k'} | /\delta ) \over \train  } \right\}$ 
				
				\STATE Set $\tol := 2 \balpha + 2\bzeta + \omega_{\train, \delta /4M} (\FF_k) $ \label{line::tolerance}

					\STATE Minimize squared loss on training set for  all $h \in [H]$ with regression targets from class $k$:
					\vspace{-1mm}
					\begin{align}
					 g_h \leftarrow  \argmin_{g \in \FF_{k'}} \quad \hat L_h(g, f_{h + 1}) := {1 \over \train} \sum_{(x, a, r, x') \in D_{\text{train}, h} }  \left( g(x, a)  - r -  f_{h + 1} (x') \right)^2
					 	\end{align}	
					\vspace{-2mm}
					\STATE Compute squared loss using the validation set for all $h \in [H]$ as a function of $f$:
					\vspace{-1mm}
					\begin{align}\label{eq::validation-loss}
					\tilde L_h (f, f_{h + 1}) = {1 \over \valid } \sum_{(x, a,r,x') \in D_{\text{valid}, h} } \left(  f(x_h, a_h) -  r_h -  f_{h + 1}(x') \right)^2
					\end{align}	
					\vspace{-2mm}
					
				\IF{$\tilde L_h(g_h, f_{h + 1}) < \tilde L_h(f_h, f_{h + 1}) -\tol $ for any $h \in [H]$} \label{line::gen-test}

						\STATE $k \leftarrow k + 1$
						\STATE goto Line~\ref{line::while}. 
					\ENDIF
				\ENDFOR
			
				\STATE goto Line~\ref{line::return}
			\ENDWHILE
			
			\RETURN $\hat \pi = \left(\pi_{f_h} \right)_{h \in [H]}$ \label{line::return}

		\end{algorithmic}
		
	\end{algorithm}

\vspace{-1.05cm}
\end{figure}

\subsection{Full Algorithm}

The full algorithm, \alg{} (Model Selection via Bellman Error), is presented in Algorithm~\ref{alg::theory}. While the underlying principle described just above is similar, \alg{} must handle a number extensions that complicate the algorithm such as dealing with general $M$, accounting for proper estimation errors, and being  robust  to the case when $k_*$ does not exist. Interestingly, the fundamental algorithmic idea remains the same -- only the tolerances change and it loops over the model classes.

\alg{} takes as input a base offline RL algorithm (such as FQI), the model classes $\FF_1 \subset \ldots \subset \FF_M$, and the offline dataset $D\in [H]$. The dataset is  split randomly into a training set $D_{\text{train}}$ and a validation set $D_{\text{valid}}$.  The algorithm begins optimistically, starting with the candidate model class $k = 1$ and running the base algorithm with $\FF_k$ on the training dataset to generate the candidate functions $f$. We retrain on the empirical square loss using a class $k'>k$ by regressing to target values $r+f_{h+1}(x')$. This amounts to solving a sequence of $H$ least squares regression problems using class $k'$, yielding the functions $g_h$.

Since $f_h$ and $g_h$ are attempting to solve the \emph{same} regression problem (with the same target values), we can compare their performance on this shared squared loss objective $\tilde L$ with validation data. We use a \textit{generalization error test} in Line~\ref{line::gen-test} to decide whether to keep using class $k$.
If the test fails and it is discovered that the larger model class $\FF_{k'}$ is able to achieve substantially smaller loss than $\FF_k$, then we move to a  larger model class $k \leftarrow k + 1$. The process is repeated until all classes are exhausted or no model class $k'$ offers a big enough improvement over $k$ to cause the test to fail.

\subsection{Rate-Optimal Oracle Inequalities}

We show that this simple procedure is able to achieve both of the oracle inequalities of the previous section simultaneously. We start with a generic version of the theorem that is stated in terms of an assumed performance bound $\omega$ on the base algorithm. We will presently instantiate the base algorithm with FQI, showing that this version precisely achieves the desired oracle inequalities with the correct rates.

\begin{definition}\label{def::base-alg}
	Let $\BB$ be a base offline RL algorithm for value function approximation that takes as input a model class $\FF$, an offline dataset $D$ of $n$ samples for each $h \in [H]$, and a failure probability $\delta$. For $\beta > 0$ and a function $\omega$, we say that $\BB$ is $(\beta, \omega)$-regular if (1)  $\omega$ is a known real-valued function of $n \in \N$, $\delta \in \R$, and $\FF_k$, and it satisfies $\omega_{n, \delta} (\FF_k) \leq \omega_{n, \delta}(\FF_{k'})$ for all $k' \geq k$; (2) $\BB(D, \FF_k, \delta)$ returns $(f_h)_{h \in [H]} \subset \FF_k$ such that $f_{h +1}$ is independent of $D_h$ and
	\begin{align}
	P \bigg(  \max_{h \in [H]}  \| f_h - T^*_h f_{h + 1} \|_{\mu_h}^2 \leq \beta \cdot \appr(\FF_k) + \omega_{n, \delta}(\FF_k) \bigg) \geq 1 - \delta.
	\end{align} 
\end{definition}
In this definition, $\beta$ represents a multiplicative factor of error on the approximation error and $\omega$ represents the estimation error, which we expect to decrease in $n$ and increase in the complexity of the class $\FF$.
Generally, we will have $\omega_{n, \delta}(\FF) = \tilde \OO\left( { \log (|\FF| /\delta) / n  }\right) $ (see Lemma~\ref{lem::fqi} for FQI). For model selection, we thus hope to achieve a bound that matches what the base algorithm would achieve had $k_*$ been known in advance, up to additive terms of $\sqrt{ \log | \FF_{k_*} | \over n }$.

Our primary theorem addresses Problem~\ref{prob::min-complete} using an arbitrary base algorithm.

\begin{restatable}{theorem}{thmmain}
\label{thm::main}
	Let $\BB$ be an $(\beta, \omega)$-regular algorithm and suppose that $k_*$ (defined in Problem~\ref{prob::min-complete}) exists. Then Algorithm~\ref{alg::theory} with inputs $D$, $\BB$, $\FF_1 \subset \ldots \subset \FF_M$, $\omega$, and $\delta \leq 1/e$ outputs $\hat \pi$ such that, with probability at least $1 - \delta$,
	\begin{align}
	\regret(\hat \pi) \leq  C \cdot  \sqrt{  \CC(\mu) H \left( \omega_{\train, \delta /4M} (\FF_{k_*})  + {H^2(\log |\FF_{k_*}|  + \iota  ) \over n} \right) }  
	\end{align}
	for some absolute constant $C > 0$ and $\iota  = \log(M^2 H /\delta)$.
	
\end{restatable}

The above theorem shows a regret bound scaling with the square root of the error term $\omega$ of the base algorithm $\BB$ plus  a $\tilde \OO(\log(|\FF_{k_*}|)/n)$ estimation error. Importantly, as stated in Problem~\ref{prob::min-complete}, the statistical complexity depends only on $\FF_{k_*}$ and not any of the larger classes.

For concreteness, we now instantiate Theorem~\ref{thm::main}  with a standard finite-horizon FQI \citep{duan2020minimax} base algorithm, which satisfies Definition~\ref{def::base-alg} with $\omega_n(\FF) = \hat \OO( \log|\FF|/n)$. This in turn translates to the desired rate-optimal oracle inequalities.

\begin{restatable}{lemma}{lemfqi}
\label{lem::fqi}
	Consider the FQI algorithm (stated in Appendix~\ref{app::fqi} for completeness). For a model class $\FF$, FQI is a $(3, \omega)$-regular base algorithm with 
	\begin{align*}
	    \omega_{n, \delta}(\FF) = \OO \left( { H^2 \log ( H|\FF| /\delta ) \over n } \right).
	\end{align*}
	
\end{restatable}

By plugging this classic result in Theorem~\ref{thm::main} as the base algorithm, we arrive at a solution to Problem~\ref{prob::min-complete}.

\begin{restatable}{corollary}{corfqi}
\label{cor::fqi}
	Let $\BB$ be instantiated with FQI (Algorithm~\ref{alg::fqi} in Appendix~\ref{app::fqi}). Define $\iota = \log(M^2 H /\delta)$ Then, under the same conditions as Theorem~\ref{thm::main}, there is an absolute constant $C > 0$ such that, with probability at least $1 -\delta$, Algorithm~\ref{alg::theory} outputs $\hat \pi$ satisfying 
\begin{align}
\regret(\hat \pi) \leq C \cdot   \sqrt{ \CC(\mu) H^3(\log |\FF_{k_*}|  + \iota  ) \over n } .
	\end{align}

\end{restatable}

The proof of Theorem~\ref{thm::main} (and by extension Corollary~\ref{cor::fqi}) follows a nearly identical intuition as outlined at the beginning of this section. In particular, the proof shows two parts: (1) \alg{} will never return a value of $k$ that exceeds $k_*$ and (2) if \alg{} returns $k < k_*$, then the approximation error must be small because it was undetectable by the test when comparing to $k_*$. 
however, a key novelty is recognizing that the generalization test in Line~\ref{line::gen-test}, which compares the errors of the two model classes on the same regression problem, can be used to prove both (1) and (2).

\subsubsection{Robustness}

We show that the same Algorithm~\ref{alg::theory} simultaneously achieves the desired robustness result of Problem~\ref{prob::robust} when $k_*$ does not exist without any modification.

\begin{restatable}{theorem}{thmrobust}
\label{thm::robust}
Under the same conditions as Theorem~\ref{thm::main}, if $k_*$ does not exist, there exists an absolute constant $C > 0$ such that, with probability at least $1 - \delta$,  Algorithm~\ref{alg::theory} outputs $\hat \pi$ satisfying
\begin{align}
	\regret(\hat \pi) \leq C \cdot \min_{k \in [M] } \left\{    \sqrt{ \CC(\mu)  H \left(  \beta \cdot \xi_k +   \omega_{\train, \delta/4M} (\FF_k)  + { H^2(\log |\FF_k|  + \iota  ) \over n }\right)   } \right\}.
	\end{align}
\end{restatable}

We can use Lemma~\ref{lem::fqi} to see a solution to Problem~\ref{prob::robust} with an instantiation of FQI.

\begin{restatable}{corollary}{corfqirobust}
\label{cor::fqi-robust}
	Under the same conditions as Corollary~\ref{cor::fqi}, there is an absolute constant $C > 0$ such that, with probability at least $1 -\delta$, Algorithm~\ref{alg::theory} outputs $\hat \pi$ satisfying 
 \begin{align}
 {
	\regret(\hat \pi) \leq  C \cdot \min_{k \in [M] } \left\{ \sqrt{\CC(\mu) H  \xi_k}   + \sqrt{ \CC(\mu) H^3 (\log |\FF_k|  + \iota  ) \over n }   \right\} }
\end{align}
	
\vspace{-1mm}

\end{restatable}

Crucially, the guarantees that solve Problems~\ref{prob::min-complete} and~\ref{prob::robust} are achieved simultaneously, meaning that we do not require knowledge of whether $k_*$ exists and we can automatically get the best of both guarantees.

 The proof of Theorem~\ref{thm::robust} (Corollary~\ref{cor::fqi-robust}) is more involved than that of Theorem~\ref{thm::main}. Rather than showing that the $k$ returned by \alg{} never exceeds the index attaining the minimum, we allow $k$ to exceed it sometimes. To ensure that the error can still be bounded,  we use the fact that class $(k - 1)$ must have failed the generalization test against some larger class in $[k, M]$. Using this fact, we can argue that the estimation error of the larger class can be bounded in terms of the unknown $\xi_{k -1}$, which we know is small since $k$ exceeds the index of the minimal class. Like before, these arguments are made possible by the generalization test in Line~\ref{line::gen-test}.

\paragraph{Computational Complexity} \alg{} is computationally efficient given a squared loss regression oracle. Within inner and outer loops over the model classes, a squared loss minimizer is computed on the training dataset and then functions are evaluated on the validation set. \alg{} requires only $\OO(H k_*M)$ calls to the computational oracle when $k_*$ exists (a consequence of Theorem~\ref{thm::main}) or $\OO(HM^2)$ in the worst case. Note that algorithms for optimizing squared loss regression problems are ubiquitous in machine learning  \citep{simchi2021bypassing}.

	\section{Empirical Results}
	\label{sec:experiments}
	\vspace{-.2cm}
	The previous sections outlined the strong theoretical properties of \alg{}. In this section,  we ask: what practical insights can be gleaned from \alg{} and its theoretical guarantees? We would like to understand if the core selection method of \alg{} can be applied out-of-the-box on existing offline RL algorithms with minimal effort.
    We evaluated \alg{} in three simulated environments with discrete actions: (1) synthetic contextual bandits (CB), (2) Gym CartPole, (3) Gym MountainCar.  See Appendix~\ref{app::exp} for specific details about the setups. All training and validation sets were split 80/20.
	
	\begin{figure}
    \centering
    \includegraphics[width=1.8in]{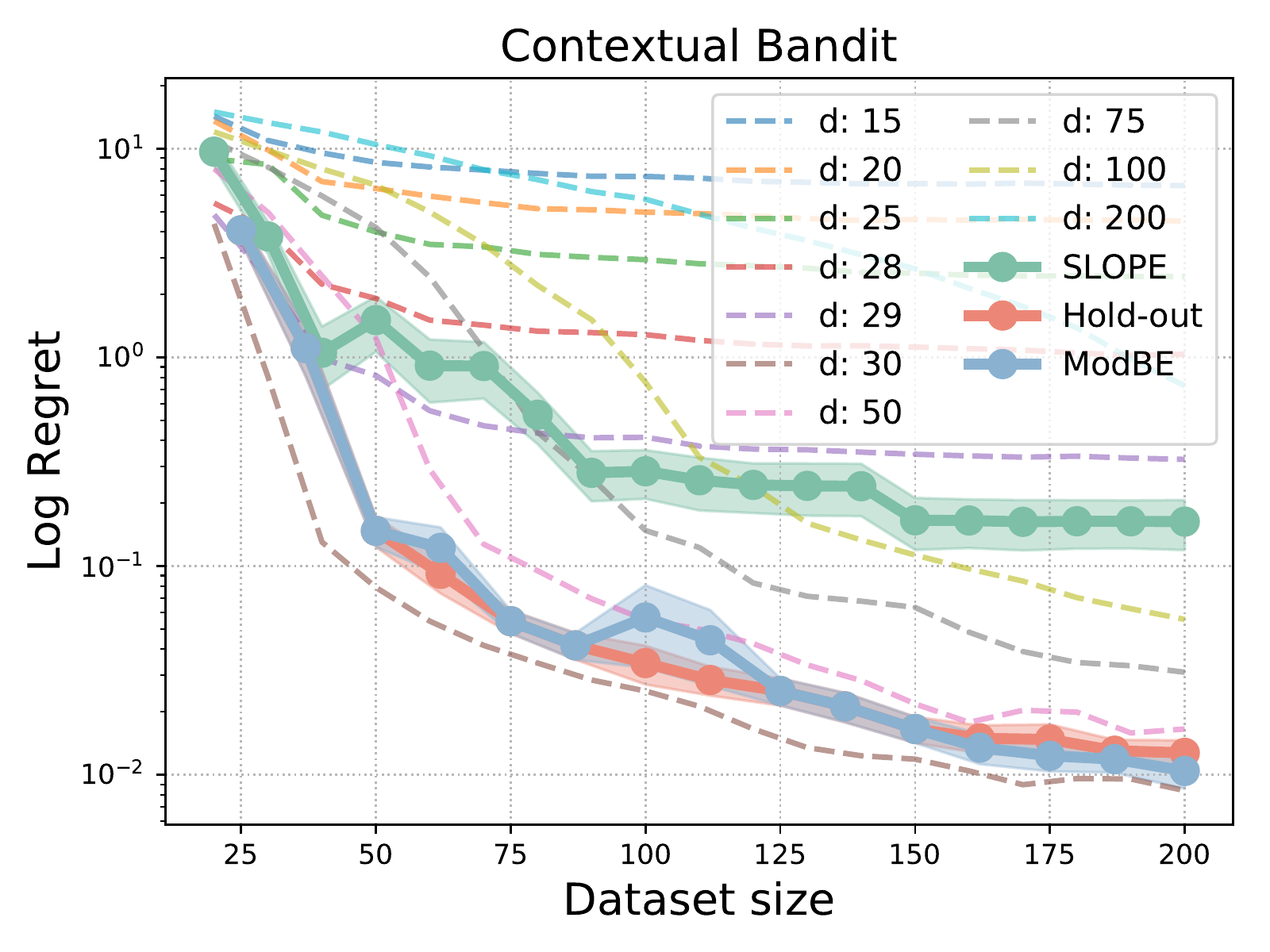}
        \includegraphics[width=1.8in]{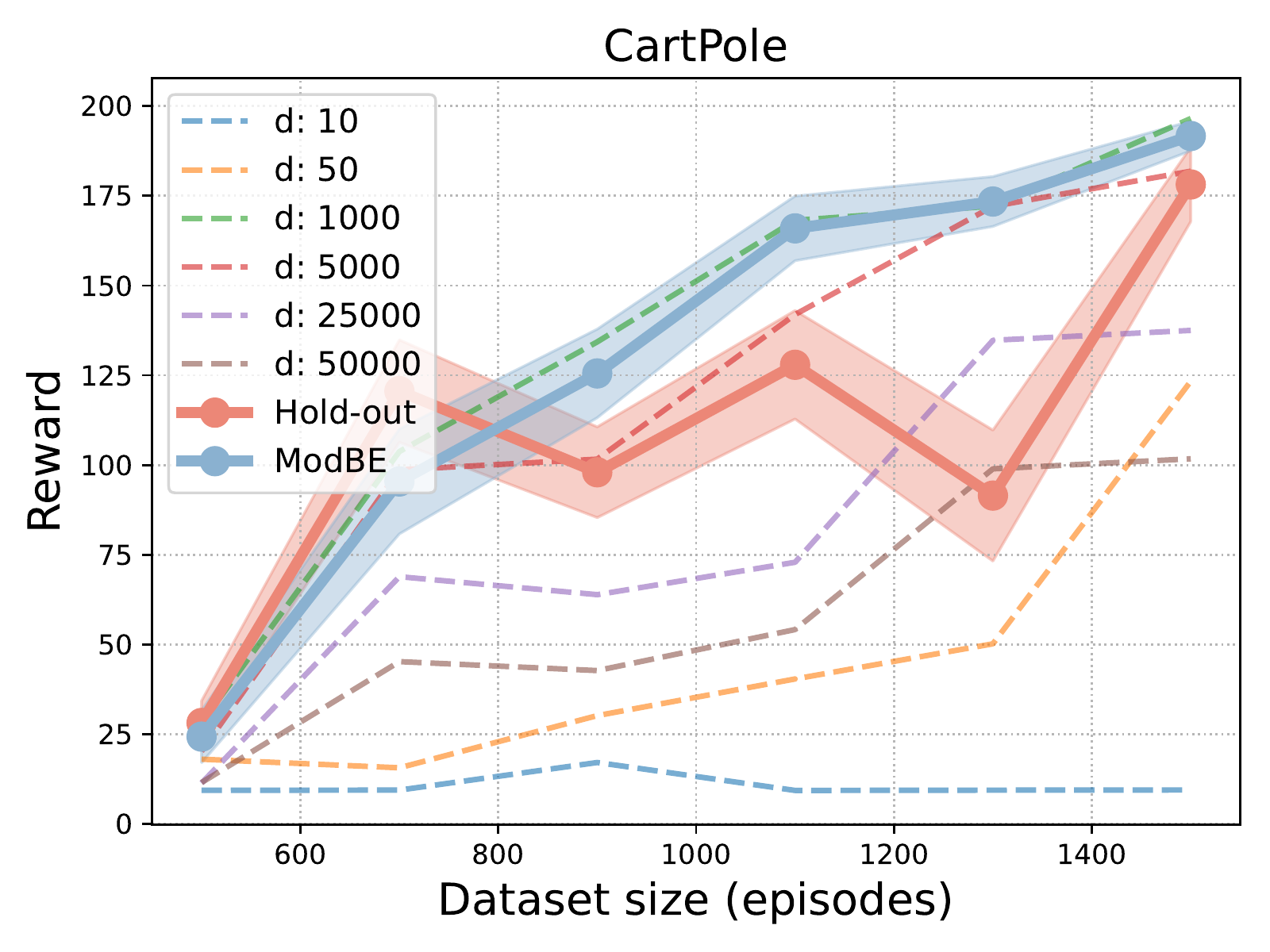}
        \includegraphics[width=1.8in]{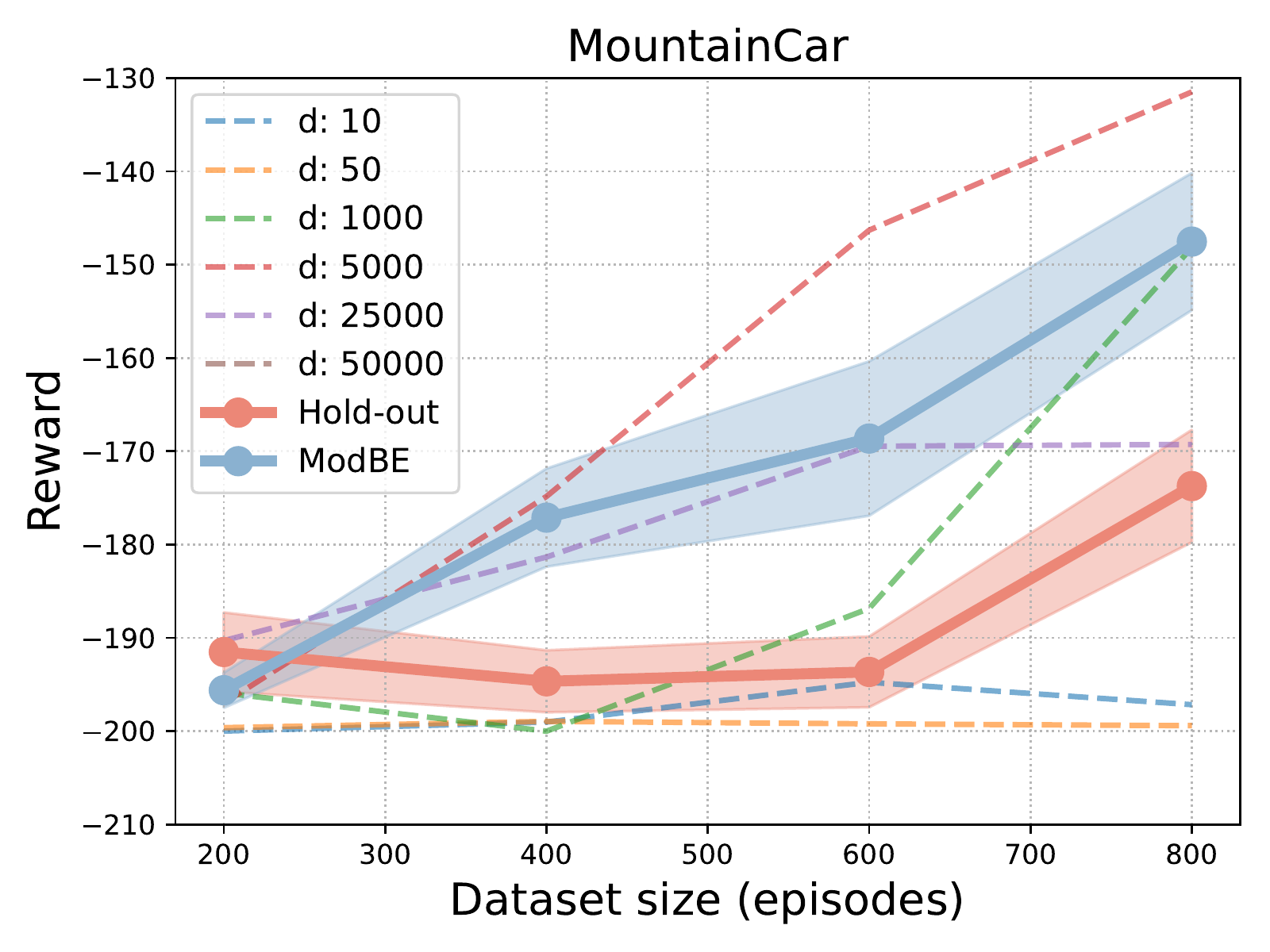}
    \caption{\alg{} is evaluated on several simulated domains: a contextual bandit (left), CartPole (middle), and MountainCar (right).  In CB, \alg{} and Hold-out outperform SLOPE and match performance of the best model class in regret. In CartPole, both match the performance of the best model class. In MountainCar, both struggle to match the best model class, but \alg{} maintains  superior performance. In CB, error bands are standard error over 10 random trials. In RL, error bands are standard error over 20 random trials.}
    \label{fig::test}
	\vspace{-.5cm}
\end{figure}

	 \paragraph{Contextual Bandit} As a basic validation experiment, we started with the CB setting of \cite{lee2021model} which considers a nested sequence of linear model classes with increasing dimension $d$. %
	Without any tuning, we simply set the tolerance of \alg{} to $\tol(\FF_k, \FF_{k'}) = {d_{k'}\over n}$. %
	Figure~\ref{fig::test} shows the results in terms of the $\log$-regret as a function of the dataset size. We observe that both \alg{} and Hold-Out (choosing the model class with the smallest error) are able to easily match the performance of the best model class while SLOPE~\citep{lee2021model} ends up being fooled by nearby classes.

	\paragraph{RL Discrete Control} Our experimental setup for the RL problems in Gym \citep{brockman2016openai} builds on top of the \textsf{d3rlpy} framework \citep{seno2021d3rlpy}, which contains open-source implementations of offline RL algorithms. We used DQN \citep{mnih2015human} (which is closest to FQI). In both CartPole and MountainCar, we considered model classes that were two-layer neural networks with ReLU activations and $d$ nodes in the hidden layer and varied the parameter $d$. Again, we simply set the tolerance of \alg{} to $d_k/n$ motivated by pseudodimension bounds~\citep{bartlett2019nearly}. For simplicity, we modified \alg{} to work in the discounted infinite horizon setting, which can trivially be done (see Appendix~\ref{app::exp} for details on this modification). The neural network classes considered had $d \in \{10, 50, 1000, 5000, 25000, 50000\}$. In both settings, we compared \alg{} to Hold-Out, which is a seemingly sensible baseline that chooses the model class with lowest estimated Bellman error on a validation set. For deterministic settings only, this is theoretically justified. Figure~\ref{fig::test} shows the reward as a function of the dataset size (in episodes). On CartPole, \alg{} and Hold-Out are both able to compete with the best classes and are roughly at parity. However, on MountainCar, we find that Hold-Out does surprisingly poorly while \alg{} is successfully able to reject the poor model classes. We conjecture that the empirical failure of Hold-out (which is not predicted in theory since the environment is deterministic) is possibly due to sensitivity to optimization error that makes the inherent Bellman error misleading. In contrast, the generalization test of \alg{} seems to be more robust  to this.
	\vspace{-.3cm}

	\section{Discussion}
	\vspace{-.3cm}
	
	In this paper, we introduced a new algorithm, \alg{}, for model selection in offline RL: to our knowledge it is the first  
	to achieve rate-optimal oracle inequalities in $n$ and $\comp(\FF_{k_*})$.
	A number of interesting open questions remain.  (1) Are there rate-optimal procedures that can be used to select hyperparameters beyond model complexity such as learning rates, batch sizes, \textit{et cetera}? (2) Can the ideas of \alg{} be extended to more general algorithms that do not rely on Bellman error minimization? (3) For the robustness guarantee, the global completeness $\xi$ is potentially much worse than $\appr(\FF)$. Is it possible to achieve a robust oracle inequality of the form $\OO(\min_{k} \sqrt{\appr(\FF_k) + { \log |\FF_k| / n}})$ when $k_*$ does not exist? We believe these questions are of great practical and theoretical importance for understanding how to effectively evaluate and select models in offline RL. 
	
\begin{ack}
We thank Annie Xie and Yannis Flet-Berliac for help and advice with experiments and anonymous reviewers for their valuable feedback. JNL is supported by the NSF GRFP. This work was also supported in part by NSF Grant \#2112926.
\end{ack}

\bibliographystyle{plainnat}
\bibliography{main}	
	\raggedbottom

\section*{Checklist}

\begin{enumerate}

\item For all authors...
\begin{enumerate}
  \item Do the main claims made in the abstract and introduction accurately reflect the paper's contributions and scope?
    \answerYes{}
  \item Did you describe the limitations of your work?
    \answerYes{} See Introduction, Discussion, and discussions of assumptions and after every theorem statement.
  \item Did you discuss any potential negative societal impacts of your work?
    \answerNA{} The work is theoretical in nature.
  \item Have you read the ethics review guidelines and ensured that your paper conforms to them?
    \answerYes{}
\end{enumerate}

\item If you are including theoretical results...
\begin{enumerate}
  \item Did you state the full set of assumptions of all theoretical results?
    \answerYes{} See preliminaries and before theorem statements.
        \item Did you include complete proofs of all theoretical results?
    \answerYes{} See appendix.
\end{enumerate}

\item If you ran experiments...
\begin{enumerate}
  \item Did you include the code, data, and instructions needed to reproduce the main experimental results (either in the supplemental material or as a URL)?
    \answerYes{}
  \item Did you specify all the training details (e.g., data splits, hyperparameters, how they were chosen)?
    \answerYes{} See Empirical Results section and Appendix~\ref{app::exp}.
        \item Did you report error bars (e.g., with respect to the random seed after running experiments multiple times)?
    \answerYes{}
        \item Did you include the total amount of compute and the type of resources used (e.g., type of GPUs, internal cluster, or cloud provider)?
    \answerYes{} See Appendix~\ref{app::exp}.
\end{enumerate}

\item If you are using existing assets (e.g., code, data, models) or curating/releasing new assets...
\begin{enumerate}
  \item If your work uses existing assets, did you cite the creators?
    \answerYes{}
  \item Did you mention the license of the assets?
    \answerYes{}
  \item Did you include any new assets either in the supplemental material or as a URL?
    \answerNA{}
  \item Did you discuss whether and how consent was obtained from people whose data you're using/curating?
    \answerNA{}
  \item Did you discuss whether the data you are using/curating contains personally identifiable information or offensive content?
    \answerNA{}
\end{enumerate}

\item If you used crowdsourcing or conducted research with human subjects...
\begin{enumerate}
  \item Did you include the full text of instructions given to participants and screenshots, if applicable?
    \answerNA{}
  \item Did you describe any potential participant risks, with links to Institutional Review Board (IRB) approvals, if applicable?
    \answerNA{}
  \item Did you include the estimated hourly wage paid to participants and the total amount spent on participant compensation?
    \answerNA{}
\end{enumerate}

\end{enumerate}
	
	\pagebreak

	\appendix

	\section{On the Nestedness of Model Classes}\label{app::hardness}
	
	Throughout this work, we assume that the given model classes are nested in the sense that $\FF_1 \subset \ldots \subset \FF_M$. While this is a very common problem setting in both supervised learning and reinforcement learning, it can cause theoretical issues with completeness since adding functions to a model class can actually increase in the completeness error unlike realizability-based assumptions. Despite this, the following proposition shows that no type of model selection bound is possible without this nestedness in general.
	
	\begin{proposition}
	   There exists a family of MDPs, a collection of model classes $\FF_1$ and $\FF_2$, and a data distribution $\mu = \mu_1, \ldots, \mu_H$  with $\CC(\mu) = \Theta(1)$ such that (1) either $\FF_1$ or $\FF_2$ complete for any MDP in the family and $|\FF_i| = \Theta(1)$ for $i \in \{1, 2\}$, and (2)
	   any algorithm that outputs $\hat \pi$ with $v(\pi^*) - \E\left[ v(\hat \pi) \right] \leq C_1 $ must use at least $C_2 S^{1/3}$ samples where $0 < C_2$ and $0< C_1 < 1$ are constants.
	\end{proposition}
	
	\begin{proof}
	The proof is a simple consequence of a recent impossibility result from \cite{foster2021offline} for problems with concentrability but without completeness. Theirs is summarized as follows:
	
	Let $H = 3$.
	Let $\FF_1 = (\FF_1^h)_{h \in [H]}$ and $\FF_2 = (\FF_2^h)_{h \in [H]}$ be time-varying model classes such that each contains a single function, $f_1^h$ and $f_2^h$ respectively for each $h$. The result of \cite{foster2021offline} shows that, defining $\FF^h = \FF_1^h  \cup \FF_2^h$, there are a family of MDP models $\MM$ and functions $f_1^h$ and $f_2^h$ such that (1) the value function $Q^\pi_h$ of any policy $\pi$ is realized in $\FF^h$ for all $h$ and any algorithm that outputs $\hat \pi$ with $v(\pi^*) - \E\left[ v(\hat \pi) \right] \leq C_1$ must use at least $C_2 S^{1/3}$ samples for constants $C_1, C_2$. 
	
	Thus, for any MDP model in $\MM$, either $\FF_1$ or $\FF_2$ satisfies completeness by the realizability condition above. A model selection oracle inequality should then ideally yield $v(\pi^*) - \E\left[ v(\hat \pi) \right] = \tilde \OO(\sqrt{1/n})$ since $|\FF_1^h| = 1$ and $|\FF_2^h| = 1$. However, this would contradiction the lower bound result of \cite{foster2021offline}, which requires at least $\Omega(S^{1/3})$ samples to achieve constant error.
	\end{proof}

	Though the argument is simple, we remark on its significance. Our model selection objectives outlined in Section~\ref{sec::problem} suggest that we should aim to achieve $v(\pi^*) - \E\left[ v(\hat \pi) \right] = \tilde \OO \left(\sqrt{ \CC(\mu) \log \FF_{k_*} \over n }\right)$ for any of the MDPs in the family, where $n$ is the number of samples in the dataset and $k_*$ is the index of the class that is complete for the given MDP. This is because it is guaranteed in the first condition that at least one of the classes is complete and realizes $Q^*$.
	
	However, the proposition shows that we will need at least $\poly(S)$ samples to achieve any non-trivial regret bound, precluding the model selection objective since $S$ can be much larger than $\log \FF_{k_*}$.
	The proposition also ensures that $|\FF_i| = \Theta(1)$ (in fact the size of both in the proof is simply $|\FF_i| = 1$), which is to say that the hardness is not due to the inherent complexity of the model classes. 
	
	Note that BVFT \citep{xie2021batch} does not contradict this hardness result for the same reason that it does not contradict the result of \cite{foster2021offline}: BVFT leverages a stronger coverage assumption.

	\section{Proof of Theorem~\ref{thm::main}}\label{app::main-proof}
	
	\subsection{Proof Sketch}\label{app::proof-sketch}

	  Having outlined the intuition behind the algorithm and generalization error test in Section~\ref{sec::algorithm}, we will now sketch the proof in a simplified setting so that the primary mechanism can be seen in a slightly more formal way. We restrict the sketch to the setting where $M = 2$ and $k_* \in \{1, 2\}$ exists. For such a setting the definition of $\tol_{n}(\FF_1, \FF_2)$ will be excessively large, but sufficient nonetheless to prove our desired oracle inequality. We define the following quantities:
	  \begin{align*}
	      L_h(f,g) & := \E_{\mu_h} \left( f(x, a) - r- g(x') \right)^2 \\
	      L_h^*(g) & := \inf_{f}  L_h(f, g)
	  \end{align*}
	  where the $\inf$ in the second line is over all measurable functions. Note that this makes $L_h^*(g)$ the \textit{irreducible} error of the regression problem which is actually achieved by $f = T^*g$.
	For the purposes of exposition, we will take $\train = n$ and assume that the validation error $\tilde L_h(f, g)$ exactly equals its expectation $
	L_h(f, g) := \E_{\mu_h}  \left( f(x, a) - r  -  g(x') \right)^2$. For the sketch only, we will ignore dependence on $H$ and we will also assume that all necessary concentration inequalities hold with high probability\footnote{Eventually, in the main proof, care will have to be taken to ensure these events to occur with high probability at the expense of logarithmic factors.}. That is, the base algorithm returns functions $f^k = (f^k_h)$ such that
	\begin{align}\label{eq::sketch-base-alg-bound}
	    \| f^k_h - T^*_h f^k_{h + 1} \|^2_{\mu_h} \leq  L \times \appr(\FF_k) +  \omega_{n, \delta}(\FF_k)
	\end{align}
	and the empirical minimizers $g^2 = (g^2_h)$ in Algorithm~\ref{alg::theory} satisfy
	\begin{align*}
	    L_h(g_h^{2}, f^1_{ h +1}) - L_h^* ( f^1_{ h +1}) & \leq \tilde \OO\left( \appr(\FF_2)  + { \log |\FF_2 | \over n } \right) 
	\end{align*}

	 We will now break the analysis down into cases:
	\begin{itemize}
	    \item If it happens that $k_* = 1$ then we will show that the test will not fail and the correct class $k_*= 1$ will always be returned by Algorithm~\ref{alg::theory}. Note that in this case $\appr(\FF_1) = 0$ by definition of $k_*$. Thus, \eqref{eq::sketch-base-alg-bound} implies that for all $h \in [H]$:
	\begin{align*}
    L_h(f^1_h, f^1_{h + 1})  - L^*_h(f^1_{h + 1}) = \| f^1_h - T^*_h f^1_{h +1}\|_{\mu_h}^2 
	& \leq \omega_{n, \delta}(\FF_1)
	\end{align*}
	
	Then, an algorithm that reliably picks $k_* = 1$ should be able to \textit{tolerate} generalization error
	 on the order of at least $\omega_{n, \delta}(\FF_1)$.  This motivates our definition of $\tol_n(\FF_1, \FF_2)$ in Line~\ref{line::tolerance} of Algorithm~\ref{alg::theory}, which ensures that $\tol_n(\FF_1, \FF_2) \geq \omega_{n, \delta}(\FF_1)$\footnote{Factors on $\delta$ (due to union bounds to handle the high probability events) are omitted in the sketch for clarity.}. Then, when the algorithm reaches the generalization test in Line~\ref{line::gen-test}, it will compare  the error in $L_h(f^1_h, f^1_{h + 1})  - L^*_h(f^1_{h + 1})$ to the error in $L_h(g^2_h, f^1_{h + 1}) - L^*_h(f^1_{h + 1})$. Since  $\tol_n(\FF_1, \FF_2) \geq \omega_{n, \delta}(\FF_1)$ and $L_h(f^1_h, f^1_{h + 1})  - L^*_h(f^1_{h + 1}) \leq \omega_{n, \delta}(\FF_1)$ (because $k_* = 1)$, we will always have that
	\begin{align*}
	    L_h(g^2_h, f^1_{h + 1}) - L^*_h(f^1_{h + 1}) & \geq L_h(f^1_h, f^1_{h + 1}) - L^*_h(f^1_{h + 1}) - \tol_n(\FF_1, \FF_2)
	\end{align*}
	And, by adding $L^*_h(f^1_{h + 1}) $ to both sides of the above display, we see that the generalization test in Line~\ref{line::gen-test} will never fail:
	\begin{align*}
	    L_h(g^2_h, f^1_{h + 1}) & \geq L_h(f^1_h, f^1_{h + 1})  - \tol_n(\FF_1, \FF_2)
	\end{align*}
	meaning that Algorithm~\ref{alg::theory} will never make the switch from $k = 1$ to $k = 2$ when $k_* = 1$, so the correct model class is returned and the error is then trivially bounded as
	\begin{align}
	    \label{eq::kstar1}
	    \| f^1_h - T^*_h f^1_{h + 1} \|_{\mu_h}^2 \leq \omega_{n, \delta} (\FF_1)
	\end{align}
	
	\item If $k_*  = 2$ and the switch to $k = 2$ is made, then the correct model class is returned and we immediately have the error bound $\| f^2_h - T^*_h f^2_{h + 1} \|_{\mu_h}^2 \leq \omega_{n, \delta} (\FF_2)$. 
	\item However, if $k_* = 2$ and the switch is not made (meaning Algorithm~\ref{alg::theory} returns $k = 1$), we can show that the error cannot be much worse than the error of $k_*$. To do this, we will use the fact that the generalization test in Line~\ref{line::gen-test} has (wrongly) succeeded in order to bound the error of $L_h(f^1_h, f^1_{h + 1}) - L^*_h(f^1_{h + 1})$ in terms of the error of $L_h(g^2_h, f^1_{h + 1}) - L^*_h(f^1_{h + 1})$ plus additional terms due to the tolerance:
	That is, for any $h \in [H]$
	\begin{align}\begin{split}\label{eq::sketcheq1}
	\| f_h^1 - T^*_h f^1_{h + 1} \|_{\mu_h}^2  & = L_h(f^1_h, f^1_{h + 1}) - L^*_h(f^1_{h + 1}) \\ & \leq L_h(g^2_h, f^1_{h + 1}) - L^*_h(f^1_{h + 1})+ \tol_n(\FF_1, \FF_2) 
	\end{split}
	\end{align}
	Recall that since $k_* = 2$, there is no approximation error for this class so $L_h(g^2_h, f^1_{h + 1}) - L^*_h(f^1_{h + 1}) \leq \tilde \OO\left( { \log |\FF_2 | \over n} \right)$. Finally, we can apply the definition of $\tol_n(\FF_1, \FF_2)$ as well as monotonicity of $\omega_n(\FF_1) \leq \omega_n(\FF_2)$ so that \eqref{eq::sketcheq1} can further be bounded as 
	\begin{align}\label{eq::kstar2}
	\| f_h^1 - T^*_h f^1_{h + 1} \|_{\mu_h}^2 & \leq \tilde \OO\left(  { \log |\FF_2 | \over n } +  \omega_{n, \delta} (\FF_2) \right)
	\end{align}
	Therefore, since $\FF_1$ did not fail the generalization test, we can actually use this to our advantage to  say that its error is not much worse than $\FF_2$ even though $k_* = 2$.
	\end{itemize}
	
	By combining the results of the case when $k_* = 1$ (where we showed \eqref{eq::kstar1} holds) and the case when $k_* = 2$ (where we showed that \eqref{eq::kstar2} holds), we have managed to show that the index $k$ returned by Algorithm~\ref{alg::theory} will satisfy
	\begin{align*}
	 \| f_h^k - T^* f_{h + 1}^k \|_{\mu_h}^2 =  \tilde \OO\left(  { \log |\FF_{k_*} | \over n}  + \omega_{n, \delta}(\FF_{k_*} )  \right)
	\end{align*}
	Appealing to the performance difference lemma (Lemma~\ref{lem::perf-diff}), we are able to guarantee that 
	\begin{align*}
	 \regret(\hat \pi) = \tilde \OO\left( \sqrt{ \CC(\mu)   \left({ \log |\FF_{k_*} | \over n}  + \omega_{n, \delta}(\FF_{k_*} )  \right) } \right)
	\end{align*}

	\subsection{Concentration Inequalities}
	We now turn to the formal proof of Theorem~\ref{thm::main}. In order to make the analysis easier, we state another version of the algorithm, which is more notation-heavy but also more precise so that we can easily refer objects in the analysis at different indices. To be clear, the algorithms are identical -- the notation has just been augmented to include indices and other modifiers for clarity.

\begin{figure}
	\begin{algorithm}[H]
		\caption{ Model Selection via Bellman Error (\alg{}) with indexing notation }\label{alg::theory-notation}
		\begin{algorithmic}[1]
			
			\STATE \textbf{Input}: Offline dataset $D = (D_h)$ of $n$ samples for each $h \in [H]$, Base algorithm $\BB$, function classes $\FF_1\subset \ldots \subset \FF_M$, failure probability $\delta \leq 1/e$. 
			
			\STATE Let $n_{\text{train}} =\ceil{ 0.8 \cdot n}$ and $n_{\text{valid}} = \floor{0.2 \cdot n}$ and split the dataset $D$ randomly into $D_{\text{train}} = ( D_{\text{train}, h})$ of $n_{\text{train}}$ samples and $D_{\text{valid}} = (D_{\text{valid}, h} ) $ of $n_{\text{valid}}$ samples for each $h \in [H]$.	
            \STATE Set $\zeta := {  96 H^2 \log (16M^2H /\delta ) \over \valid } $.

			\STATE Initialize $k  \leftarrow 1$.

			\WHILE{$k < M$} \label{line::while2}
			
				\STATE $f^k := (f^k_h)_{h \in [H]} \leftarrow \BB(D_{\text{train}}, \FF_k, \delta/4M)$

				\FOR{$k' \leftarrow k + 1, \ldots, M$}
				    
				    \STATE Set $\balpha_k :=  \max \left\{ \omega_{\train, \delta/4M}(\FF_k), { 200 H^2 \log(8 M^2 H |\FF_{k} | /\delta ) \over \train  } \right\}$ for all $k \in [M]$
				    
				    \STATE Set $ \tol_{\train}(\FF_k, \FF_{k'}) := 2 \delta_{k'} + 2 \zeta + \omega_{\train, \delta /4M} (\FF_k) $ for all $k < k'$. \label{line::tolerance2}
				
					\STATE Minimize squared loss on training set for all $h \in [H]$ with regression targets from class $k$:
					\vspace{-1mm}
					\begin{align}
					 g^{k'}_h \leftarrow  \argmin_{g \in \FF_{k'}} \quad \hat L_h(g, f_{h + 1}^{k}) := {1 \over \train} \sum_{(x, a, r, x') \in D_{\text{train}, h} }  \left( g(x, a)  - r -  f^{k}_{h + 1} (x') \right)^2
					 	\end{align}	
					\vspace{-2mm}
					\STATE Compute squared loss estimator using the validation set for all $h \in [H]$ as a function of $f$:
					\vspace{-1mm}
					\begin{align}\label{eq::validation-loss2}
					\tilde L_h (f, f^k_{h + 1}) = {1 \over \valid } \sum_{(x, a,r,x') \in D_{\text{valid}, h} } \left(  f(x_h, a_h) -  r_h -  f^k_{h + 1}(x') \right)^2
					\end{align}	
					\vspace{-2mm}
					
				\IF{$\tilde L(g^{k'}_h, f^k_{h + 1}) < \tilde L(f^{k}_h, f^k_{h + 1}) -\tol_{\train}(\FF_{k}, \FF_{k'} )$ for any $h \in [H]$} \label{line::gen-test2}

						\STATE $k \leftarrow k + 1$
						\STATE goto Line~\ref{line::while2}. 
					\ENDIF
				\ENDFOR
			
				\STATE goto Line~\ref{line::return2}
			\ENDWHILE
			
			\RETURN $\hat \pi = \left(\pi_{f^{k}_h} \right)_{h \in [H]}$ \label{line::return2}

		\end{algorithmic}
		
	\end{algorithm}

\end{figure}

	We require several basic components in order for the final model selection bound to hold. The first few are concentration results concerning the datasets. These will allow us to prove generalization error bounds for each of the classes as well as to obtain good estimates of the regression error via the validation set.

	Recall some useful shorthand notation to represent the true and empirical loss functions. For any measurable functions $f, g \in (\XX \times \AA \to \R)$ and a training dataset $D = (D_h)_{h \in [H]}$ of $n$ samples for each $h$ and a validation dataset $D' = (D'_h)_{h \in [H]}$ of $m$ samples for each $h$, we define
	\begin{align}
	L_h(f, g) & = \E_{\mu_h}  \left( f(x, a) - r - g(x') \right)^2 \\
	L_h^*(g) & = \inf_f L_h(f,g) \\
	\hat L_h(f,g) &  = { 1\over n } \sum_{ (x,a,r, x') \in D_h }   \left( f(x_{i, h}, a_{i, h}) - r(x_{i, h}, a_{i, h}) - g(x'_{i, h})  \right)^2 \\
	\tilde L_h(f,g) & = { 1\over m } \sum_{ (x, a,r, x',) \in D'_{h}}   \left( f(x, a) - r - g(x')  \right)^2
	\end{align}
	
	where, as in the proof sketch, the $\inf$ in the second line is also over all measurable functions. 
	Finally, recall that, for any functions $f, g \in (\XX \times \AA \to \R)$, we have defined
	\begin{align*}
	    \| f -T^* g\|_{\mu_h}^2 := \E_{\mu_h} \left( f(x,a) - T^*g(x, a) \right)^2
	\end{align*}
	It is easy to see that this is equal to $L_h(f,g)$ without the irreducible erorr: $\| f - T^* g\|_{\mu_h}^2 = L_h(f, g) - L_h^*(g)$.

	To proceed with the concentration analysis, we will show that all the necessary events will hold simultaneously with high probability. This requires defining some additional notation.
	
	We let $f^k = (f^k_h)_{h \in [H]} \leftarrow \BB(D, \FF_k)$ for each $k \in [M]$. Then, for each $h \in [H]$, we define the empirical minimizers for a larger class $k'$ with the same regression target as follows:
	\begin{align*}
	    g^{k' \to k}_{h} = \argmin_{g \in \FF_{k'}} \hat L_h(g, f^k_{h  +1} )
	\end{align*}
	for all $k' \geq k$. Not all of these need be computed in the execution of Algorithm~\ref{alg::theory}, but we will analyze them all for the sake of simplicity in the concentration analysis. We define the following event that guarantees all of these value function approximators achieve their desired errors simultaneously up to log factors.
	
	\begin{align*}\EE_1 = \bigcap_{k \in [M], k' > k, h \in [H]} \left\{  \| g^{k' \to k}_h - T^*_h f^k_{h +1 }  \|_{\mu_h}^2 \leq 3 \appr(\FF_{k'}) +  {C_1 H^2 \log( 8 M^2 H \abs{\FF_{k'}} /\delta) \over n  }   \right\} \end{align*}

	where $C_1 > 0$ is a constant to be determined.
	The next event ensures that the base algorithm actually achieves its guarantees from Definition~\ref{def::base-alg}.
	\begin{align*}
	    \EE_2 = \bigcap_{k \in [M], h \in [H]} \left\{ || f^k_h - T^*_h f^k_{h + 1} \|_{\mu}^2 \leq \beta \cdot \appr(\FF_{k}) + w_{n, \delta/ 4M}(\FF_k)   \right\}
	\end{align*}
    The above $\EE_2$ occurs with probability at least $1 - {\delta \over 4}$ essentially by definition of the base algorithm. The last event that we are interested in relates the true loss $L_h$ to the validation loss $\tilde L_h$ on the independent dataset $D'$. We define $\tilde L_h^*(f_{h + 1}^k) :=\tilde L_h(T^*_hf_{h + 1}^k, f_{h + 1}^k)$. The events are given by
    \begin{align*}
        \EE_3 &  =\bigcap_{k \in [M], k' > k, h \in [H]} \left\{ \tilde L_h(g^{k' \to k}_h, f^k_{h + 1} )  - \tilde L^*_h(f^k_{h +1})  \leq 2 \left(  L_h(g^{k' \to k}_h, f^k_{h + 1} )  -  L^*_h(f^k_{h +1})  \right) + { C_3 H^2 \log(16HM^2 /\delta) \over m   }  \right\}  \\
        & \quad \bigcap \bigcap_{k \in [M], h \in [H]} \left\{ \tilde L_h(f^k_h, f^k_{h + 1} )  - \tilde L^*_h(f^k_{h +1})  \leq 2 \left(  L_h(f^k_h, f^k_{h + 1} )  -  L^*_h(f^k_{h +1})  \right) + { C_3 H^2 \log(16HM^2 /\delta) \over m   }  \right\} \\
        \EE_4 & =  \bigcap_{k \in [M], k' > k, h \in [H]} \left\{  L_h(g^{k' \to k}_h, f^k_{h + 1} )  -  L^*_h(f^k_{h +1})  \leq 2 \left(  \tilde L_h(g^{k' \to k}_h, f^k_{h + 1} )  -  \tilde L^*_h(f^k_{h +1})  \right) + {C_4 H^2 \log(16HM^2 /\delta) \over m   }  \right\}  \\
        & \quad \bigcap \bigcap_{k \in [M], h \in [H]} \left\{  L_h(f^k_h, f^k_{h + 1} )  -  L^*_h(f^k_{h +1})  \leq 2 \left(  \tilde L_h(f^k_h, f^k_{h + 1} )  -  \tilde L^*_h(f^k_{h +1})  \right) + {C_4 H^2 \log(16HM^2 /\delta) \over m   }  \right\}
    \end{align*}
    where $C_3, C_4 > 0$ are constants to be determined. We will prove the following guarantee.
    
    \begin{theorem}\label{thm::concentration}
    Let $\EE = \bigcap_{i = 1}^4  \EE_i$. Then, $P(\EE) \geq 1 - \delta$.
    \end{theorem}
	
	To prove this result, we will show that these events occur with high probability. We require several intermediate results, starting simply with Bernstein's inequality.

	\begin{lemma}[Bernstein's Inequality]\label{lem::bernstein}
	    Let $Z_1, \ldots, Z_n$ be a sequence of independent random variables with $\E [Z_i] = 0$, $\sigma^2 = \var(Z_i)$ and $|Z_i | \leq B$. Then, with probability at least $1 - \delta$, for any $\eta > 0$
	    \begin{align*}
	        | \sum_{i \in [n]} Z_i | & \leq \sqrt{ 2 n  \sigma^2 \log (2/\delta)}  + {  B \log (2/\delta) } \\
	        & \leq {  n \sigma^2 \over 2 \eta  } + { B \log (2/\delta)   } +  {  \eta \log (2/\delta ) } 
	    \end{align*}
	\end{lemma}
	
	\begin{proof}
	    The first inequality is a standard Bernstein inequality found in, for example, \citet{vershynin2018high}. The second inequality follows by applying the AM-GM inequality to the first.
	\end{proof}
	
	    The next lemma is a generalization error bound showing that the Bellman error of a function $f \in \FF$ can be bounded in terms of the excess training loss, the approximation error of $\FF$, and the estimation error which is $\tilde \OO(\log|\FF | / n )$.
		The next lemma shows that the minimizer of the empirical squared loss achieves good generalization error with respect to the optimal function in its class.

	\begin{lemma}\label{lem::minimization}
	    Fix $\FF \subset (\XX \times \AA \to [0, H])$ and $h$. Let $g \in \GG$ be fixed where $\GG \subset \FF$. For $i \in [n]$ and  $(x_i, a_i) \sim \mu_h$ and $x'_i \sim \Pr(\cdot  | x, a)$, define $y_i = r(x_i, a_i) + g(x_i')$. Define $Z_i^f =  \left(f(x_i, a_i) - y_i\right)^2 - \left( f^*(x_i, a_i)  - y_i\right)^2$ where $f^* = \argmin_{f \in \FF} \| f - T^*_h g\|_{\mu_h}^2$. Then, with probability at least $1 - \delta$, for all $f\in \FF$ simultaneously,
	    \begin{align*}
	        \| f - T^*_h g \|_{\mu_h}^2 \leq {2 \over n} \sum_i {Z_i^f} + 3  \| f^* - T^*_h g \|_{\mu_h}^2  + {40 (H + 1)^2 \log (2 \abs{\FF}/\delta) \over n}
	    \end{align*}
	\end{lemma}
	\begin{proof}
	    
	    We will drop some of the sub- and super-script notation with the understanding that $\E$ means $\E_{\mu_h}$ and $Z_i$ means $Z_i^f$. It is easy to see that $\E [ Z_i] = L(f, g) - L(f^*, g) = \| f - T^* g \|^2_{\mu_h} - \| f^* - T^* g\|^2_{\mu_h}$. Furthermore, we can bound the variance as 
	    \begin{align*}
	        \var(Z_i)  & = \E [Z_i^2] - \E[Z_i]^2  \\
	        & \leq \E Z_i^2 \\
	        &  =\E  \left(  \left(f(x_i, a_i) - y_i\right)^2 - \left( f^*(x_i, a_i)  - y_i\right)^2 \right)^2 \\
	        & = \E \left(  \left(f(x_i, a_i) - f^*(x_i, a_i)\right)^2 + 2 (f(x_i, a_i) - f^*(x_i, a_i) )(f^*(x_i, a_i) - y_i ) \right)^2 \\
	        & = \E \left( f(x_i, a_i) - f^*(x_i, a_i)  \right)^2 \left( f(x_i, a_i) +  f^*(x_i, a_i) - 2y_i  \right)^2 \\
	        & \leq 4 (H  + 1)^2 \| f - f^* \|_{\mu_h}^2 \\
	         & \leq 8  ( H + 1)^2  \left( \| f - T^* g \|_{\mu_h}^2 + \| f^* - T^* g \|_{\mu_h}^2\right) 
	    \end{align*}
	    where we have used the fact that $f(\cdot, \cdot), f(\cdot, \cdot) \in [0, H]$, $y_i \in [0, H + 1]$, and $(a + b)^2 \leq 2a^2 + 2b^2$ for $a, b \in \R$.
	    Using Lemma~\ref{lem::bernstein}, we have that with probability at least $1 - \delta$,
	    \begin{align*}
	        \| f-   T^* g \|_{\mu_h}^2 - \| f^* - T^* g \|_{\mu_h}^2 & = \E[Z_i]   \\
	        & \leq  {1 \over n } \sum_{i} Z_i + { \var(Z_1) \over \eta } + {4(H + 1)^2 \log (2/ \delta) \over n} + { \eta \log (2/\delta) \over n } \\
	        & \leq {1 \over n } \sum_{i} Z_i + { 8  ( H + 1)^2  \left( \| f - T^* g \|_{\mu_h}^2 + \| f^* - T^* g \|_{\mu_h}^2\right)  \over \eta }  \\
	        & \quad + {4(H + 1)^2 \log (2/ \delta) \over n} + { \eta \log (2/\delta) \over n } \\
	        & \leq {1 \over n } \sum_{i} Z_i + { \| f - T^* g \|_{\mu_h}^2  + \| f^* - T^* g \|_{\mu_h}^2  \over 2}  + {20(H + 1)^2 \log (2/ \delta) \over n}
	    \end{align*}
	    where in the last equality we have chosen $\eta = 16 ( H + 1)^2$. Rearranging and then taking the union bound over all $f \in \FF$ gives the result.
	    \end{proof}

	    	Note that if we take $f = \hat f_h$ to be the empirical minimizer of $\hat L_h(\cdot, g)$, then the bound in Lemma~\ref{lem::minimization} becomes
    \begin{align*}
        \| \hat f_h - T^*_h g \|_{\mu_h}^2  & \leq 3  \| f^* - T^*_h g \|_{\mu_h}^2  + {40 (H + 1)^2 \log (2 \abs{\FF}/\delta) \over n} \\
        & \leq 3 \appr(\FF)  + {40 (H + 1)^2 \log (2 \abs{\FF}/\delta) \over n}
    \end{align*}
	    where the last inequality follows because $\GG \subset \FF$. Equipped with these bounds, we are now ready to prove that event $\EE_1$ holds with good probability.
	    
	    \begin{proposition}\label{prop::event-minimizer}
	           $P(\EE_1) \geq 1- {\delta \over 4}$ with the constant $C = 200$.
	    \end{proposition}
	    
	    \begin{proof}
	        The proof follows by repeatedly applying Lemma~\ref{lem::minimization}. Note that $\hat f^k_{h + 1}$ is independent of the data $D_h$. Therefore, for any $h$ we may condition on $f^k_{h +1}$ and see that 
	        \begin{align} \label{eq::e1-intermediate}
	            \|g^{k' \to k}_h - T^* f^k_{h + 1} \|_{\mu_h}^2 \leq 3 \appr(\FF_{k'}) + {40 (H + 1)^2 \log (2 \abs{\FF}/\delta) \over n}
	        \end{align}
	        with probability at least $1 - \delta$. By this independence, integrating ensures that the above holds regardless of $f^k_{h + 1}$. Taking the union bound over all $h \in [H]$, all $k \in [M]$ and all $k' > k$, we get that \eqref{eq::e1-intermediate} holds for all with probability at least $1 - M^2 H \delta$. Changing variables to $\delta' = 4 M^2 H \delta$ completes the proof.
	    \end{proof}

	    \begin{proposition}\label{prop::event-base-alg}
	       $P(\EE_2) \geq 1 - { \delta \over 4}$.
	    \end{proposition}
	    \begin{proof}
	        This follows immediately from Definition~\ref{def::base-alg} and a union bound and changing variables $\delta' = 4M \delta$.
	    \end{proof}
	    
	    \begin{proposition}\label{prop::event-validation}
	       $P(\EE_3 \cap \EE_4) \geq 1 - { \delta \over 2} $ with $C_3 = C_4 = 96$.
	    \end{proposition}
	    
	    \begin{proof}
	        Fix a single tuple $(k, k', h)$. For shorthand, let us define $f := f^k_{h + 1}$, $g := g^{k' \to k}_h$ and $g^* := T^*_h f^k_{h + 1}$. Then, similar to the proof of Lemma~\ref{lem::minimization},  we define $y_i = r_i + f(x_i')$ and $Z_i = (g(x_i, a_i) - y_i)^2 - (g^*(x_i, a_i) - y_i)^2$.
	        
	        Note that $\E_{\mu_h} [ Z_i   ] = L(g, f) - L^*(f) = \| g - T^*_h f \|_{\mu_h}^2$. Similarly, $\var(Z_i)  \leq \E [ Z_i^2 ]$ where
	        \begin{align*}
	             \E [ Z_i^2 ] & = \E \left( g(x_i, a_i) - g^*(x_i, a_i)  \right)^2 \left( g(x_i, a_i) + g^*(x_i, a_i) - 2y_i \right)^2 \\
	             & = 4 (H + 1)^2 \E \left( g(x_i, a_i) - g^*(x_i, a_i)  \right)^2 \\
	             & \leq 4 (H + 1)^2 \| g - g^* \|_{\mu_h}^2 \\
	             & = 4 ( H + 1)^2 \left( L_h(g, f) - L_h^*(f)\right)
	        \end{align*}
	        By Lemma~\ref{lem::bernstein}, we can guarantee that
	        \begin{align*}
	             | \left(\tilde L_h(g, f) - \tilde L_h(g^*, f) \right)- \left(L_h(g, f) - L^*_h(f) \right)  |  & \leq {4 ( H + 1)^2 \left( L_h(g, f) - L_h^*(f)\right) \over \eta }  \\
	             & \quad + {4(H + 1)^2 \log(2/\delta) \over m} + { \eta \log(2/\delta) \over m} \\
	             & = { \left( L_h(g, f) - L_h^*(f)\right) \over 2 } + {12(H + 1)^2 \log(2/\delta) \over m} 
	        \end{align*}
	        with probability at least $1 - \delta$.
	        Rearranging terms, we are able to conclude that 
	        \begin{align*}
	            \tilde L_h(g, f) - \tilde L_h(g^*, f) \leq { 3 \left( L_h(g, f) - L_h^*(f)\right) \over 2 } + {12(H + 1)^2 \log(2/\delta) \over m}  
	        \end{align*}
	        and, simultaneously,
	        \begin{align*}
	           L_h(g, f) - L^*_h(f)  & \leq  2 \left(\tilde L_h(g, f) - \tilde L_h(g^*, f) \right) + {24(H + 1)^2 \log(2/\delta) \over m}
	        \end{align*}
	        We may repeat the same calculation when setting $g = f^k_h$ for all $ k \in [M]$ and $h \in [H]$.
	        Taking the union bound over all $(k, k', h)$ and changing variables to $\delta' = 4(M^2 H  + MH) \delta$ gives the result.
	    \end{proof}

	    \begin{proof}[Proof of Theorem~\ref{thm::concentration}]
	    The result follows immediately by a union bound combining the events $\EE_1$, $\EE_2$, and $\EE_3$ and $\EE_4$, where it was shown that $P(\EE_1) \geq 1 - {\delta \over 4}$, $P(\EE_2) \geq 1 - {\delta \over 4}$ and $P(\EE_3 \cap \EE_4) \geq 1 - {\delta \over 2}$.
	    \end{proof}

\subsection{Proof of Theorem~\ref{thm::main}}

Armed with the concentration results of the previous section, we are ready to prove Theorem~\ref{thm::main}, which is restated here for clarity.

\thmmain*

\begin{proof}
Let us assume the event $\EE$ holds using the training dataset $D_{\text{train}}$ of $\train$ samples and validation dataset $D_{\text{valid}}$ of $\valid$ samples. Theorem~\ref{thm::concentration} shows that $P(\EE) \geq 1-  \delta$. Recall that the training set size is $\train$ and the validation set size is $\valid$.  As shorthand, Algorithm~\ref{alg::theory-notation} also defines the following quantities:
\begin{align*}
    \balpha_k & =  \max \left\{ \omega_{\train, \delta/4M}(\FF_k), { C_1 H^2 \log(8 M^2 H |\FF_{k} | /\delta ) \over \train  } \right\} \\
    \bzeta & = {  C_3 H^2 \log (16M^2H /\delta ) \over \valid } 
\end{align*}
where $C_1$ and $C_3$ are the constants from Propositions~\ref{prop::event-minimizer} and~\ref{prop::event-validation}.

Note that $\balpha_k$ is still monotonically non-decreasing in $k$ as both sequences that comprise it are monotonically non-decreasing. Recall the definition of $\tol_{\train}(\FF_k, \FF_{k'})$:
\begin{align*}
    \tol_{\train}(\FF_k, \FF_{k'}):= 2 \omega_{\train, \delta/4M}(\FF_k) + 2 \bzeta + \balpha_{k'}
\end{align*}
We will drop the subscript notation on $\omega$ and $\tol$ with the implicit understanding that $\omega(\cdot) = \omega_{\train, \delta/ 4M}(\cdot)$ and $\tol (\cdot, \cdot) = \tol_{\train}(\cdot, \cdot)$.

We will prove the oracle inequality of Theorem~\ref{thm::main} when $k_*$ exists (second claim of Theorem~\ref{thm::main}). Consider the following cases.
\begin{enumerate}
    \item Suppose that algorithm has currently reached $k = k_*$. We can guarantee that the generalization test in Line~\ref{line::gen-test2} will never fail in this situation, and, therefore, the algorithm will return $k = k_*$ which achieves the desired oracle inequality by definition. Note that by $\EE$, for all $h \in [H]$ and $k' > k$,
    \begin{align*}
          \tilde L_h(f^k_h, f^k_{h + 1})  & - \tilde L^*_h(f^k_{h + 1}) - \tol(\FF_k, \FF_{k'} )  \\
         & \leq
        2  L_h(f^k_h, f^k_{h + 1}) - 2 L^*_h(f^k_{h + 1})  + \bzeta  -  \tol(\FF_k, \FF_{k'}) \\
        & = 2  \| f_h^k - f_{h + 1}^k \|_{\mu_h}^2 + \bzeta  -  \tol(\FF_k, \FF_{k'}) \\
        & \leq 2 \omega (\FF_k) + \bzeta  -  \tol(\FF_k, \FF_{k'}) \\
        & =  -\balpha_{k'} - \bzeta \\
        & \leq - \bzeta
    \end{align*}
    where the second inequality has used $\EE_2$ along with the fact that $\appr(\FF_k) = 0$ in this case. Similarly, we have that
    \begin{align*}
    0  & \leq \frac{1}{2} \left(L_h(g^{k' \to k}_h, f^k_{h + 1}) - L_h^*( f^k_{h + 1}) \right) \\
    & \leq \tilde L_h(g^{k' \to k}_h, f^k_{h + 1}) - \tilde L_h^*( f^k_{h + 1}) + \bzeta
    \end{align*}
    The above inequalities imply that we will always find that
    \begin{align*}
    \tilde L_h(f^k_h, f^k_{h + 1})   - \tilde L^*_h(f^k_{h + 1}) - \tol(\FF_k, \FF_{k'} )   & \leq - \bzeta \\
    & \leq \tilde L_h(g^{k' \to k}_h, f^k_{h + 1}) - \tilde L_h^*( f^k_{h + 1})
    \end{align*}
    and therefore
    \begin{align*}
        \tilde L_h(f^k_h, f^k_{h + 1})- \tol(\FF_k, \FF_{k'} )   & \leq \tilde L_h(g^{k' \to k}_h, f^k_{h + 1})
    \end{align*}
    Therefore, the test will never fail when $k = k_*$ while $\EE$ holds.
    
    \item Now let us consider the case where Algorithm~\ref{alg::theory}
 returns $k < k_*$. In this case, the test succeeded for all $k' > k$ even though class $\FF_k$ has $\appr(\FF_K) \neq 0$. It remains to show that little is lost in this case even though there is approximation error in the returned class. Note that this implies that the test succeeded for $k' = k_*$. Therefore, we have 
 \begin{align*}
    \tilde L_h(f^k_h, f^k_{h + 1})   - \tol(\FF_k ,\FF_{k'}) \leq \tilde L_h(g^{k_* \to k}_h, f^k_{h + 1})
 \end{align*}
 for all $h \in [H]$. Then, event $\EE$ implies that
 \begin{align*}
    \| f_h^k - T^*_h f^k_{h + 1} \|_{\mu_h}^2 & = L_h(f_h^k, f^k_{h + 1} ) - L_h^*( f^k_{h + 1} ) \\
    & \leq 2 \left( \tilde L_h(f_h^k, f^k_{h + 1} ) - \tilde  L_h^*( f^k_{h + 1} ) \right) + \bzeta \\
    & \leq 2 \left( \tilde L_h(g^{k_* \to k}_h, f^k_{h + 1}) - \tilde  L_h^*( f^k_{h + 1} ) \right) + \bzeta + 2 \tol(\FF_k, \FF_{k_*})  \\
    & \leq 4 \left( L_h(g^{k_* \to k}_h, f^k_{h + 1}) -  L_h^*( f^k_{h + 1} ) \right) + 2 \bzeta + \bzeta + 2\tol(\FF_k, \FF_{k_*}) \\
    & \leq 8 \balpha_{k_*}  + 7 \bzeta + 2 \omega (\FF_k) \\
    & \leq 8 \balpha_{k_*}  + 7 \bzeta + 2 \omega (\FF_{k_*})
 \end{align*} 
 where the second to last line follows from applying $\EE_1$ along with the fact that $f_{h + 1}^k \in \FF_{k} \subset \FF_{k_*}$ and the last line uses the monotonicity property $\omega(\FF_k) \leq \omega(\FF_{k_*})$ since $k < k_*$ by assumption.
 
 \end{enumerate}

Since all the cases have been handled, we see that we are able to guarantee that, for all $h \in [H]$
\begin{align*}
    \| f_h^k - T^*_h f^k_{h + 1} \|_{\mu_h}^2 & \leq 8\balpha_{k_*}  + 7 \bzeta + 2 \omega (\FF_{k_*})
\end{align*}
Appealing to the performance difference lemma, the regret can be bounded as
\begin{align*}
    \regret(\hat \pi) & \leq 2 \sqrt{ \CC(\mu)  H \left( 8 \balpha_{k_*}  + 7 \bzeta + 2 \omega (\FF_{k_*}) \right)   } 
\end{align*}

This completes the proof of the second claim of Theorem~\ref{thm::main} when $k_*$ exists. 
\end{proof} 
\subsection{Proof of Theorem~\ref{thm::robust}}

\thmrobust*

\begin{proof}
Now consider the case where $k_*$ does not necessarily exist. This setting is slightly more challenging as we must tolerate the case where Algorithm~\ref{alg::theory} outputs $k$ that is too large; whereas, in the previous case, we showed that such an event could never occur. Let us denote $k^\dagger = \argmin_{k \in [M]}\left\{ \xi_k + \omega_{\train, \delta/4M}(\FF_k)   \right\}$. 

\begin{enumerate}
    \item If the algorithm returns $k = k^\dagger$, then we are done.
    
    \item Consider the case where $\FF_k$ is returned with  $ k < k^\dagger $. Then, since the test has succeeded with $k^\dagger$, we have that for all $h \in [H]$
    \begin{align*}
        \tilde L_h(f_h^k, f^k_{h + 1} ) - \tilde  L_h^*( f^k_{h + 1} ) & \leq \tilde L_h(g_{h}^{k^\dagger \to k}, f^k_{h + 1} ) - \tilde  L_h^*( f^k_{h + 1} ) + \tol(\FF_k, \FF_{k^\dagger}) \\
        & \leq 2 \left( L_h(g_{h}^{k^\dagger \to k}, f^k_{h + 1} ) -  L_h^*( f^k_{h + 1} )\right) + \bzeta  + \tol(\FF_k, \FF_{k^\dagger}) \\
        & \leq 2 \left(3 \appr(\FF_{k^\dagger}) + \delta_{k^\dagger} \right) + \bzeta + \tol(\FF_k, \FF_{k^\dagger})
    \end{align*}
    Furthermore, 
    \begin{align*}
          L_h(f_h^k, f^k_{h + 1} ) -   L_h^*( f^k_{h + 1} ) & \leq 2 \left( \tilde L_h(f_h^k, f^k_{h + 1} ) -   \tilde L_h^*( f^k_{h + 1} ) \right)  + \bzeta  \\
          & \leq 12 \appr(\FF_{k^\dagger})  + 4 \balpha_{k^\dagger}  + 3 \bzeta + 2\tol(\FF_k, \FF_{k^\dagger}) \\
          & \leq 12 \appr(\FF_{k^\dagger})  + 8 \balpha_{k^\dagger}  + 7 \bzeta + 2 \omega(\FF_k) \\
          & \leq 12 \appr(\FF_{k^\dagger})  + 8 \balpha_{k^\dagger}  + 7 \bzeta + 2 \omega(\FF_{k^\dagger})
    \end{align*}
    \item Finally, we consider the last case where $\FF_{k}$ is returned for $k > k^\dagger$. This implies that for $i = k - 1$ there is some $j \in [k, M]$ and $h \in [H]$ such that the test failed. That is,
    \begin{align*}
        \tilde L_h(g_h^{j\to i}, f^i_{h + 1} ) -  \tilde  L_h^*( f^i_{h + 1} ) & \leq \tilde L_h(f_h^{ i}, f^i_{h + 1} ) -    \tilde L_h^*( f^i_{h + 1} ) - \tol(\FF_{i}, \FF_{j}) \\
        & \leq 2 \left( L_h(f_h^{ i}, f^i_{h + 1} ) -     L_h^*( f^i_{h + 1} )\right)  + \bzeta - \tol(\FF_{i}, \FF_{j}) \\
         & \leq 2 \left( \beta \cdot \appr(\FF_i) + \omega(\FF_i) \right)  + \bzeta - \tol(\FF_{i}, \FF_{j})  
    \end{align*}
    where the last line uses event $\EE_2$ from the base algorithm guarantee.
    Further lower bounding the left side, we get that
    \begin{align}\label{eq::theorem1proof-eq1}
    \begin{split}
        0  & \leq \frac{1}{2} \left(  L_h(g_h^{j\to i}, f^i_{h + 1} ) -    L_h^*( f^i_{h + 1} ) \right) \\
        & \leq \tilde L_h(g_h^{j\to i}, f^i_{h + 1} ) -   \tilde  L_h^*( f^i_{h + 1} ) + \bzeta \\
        & \leq 2 \left( \beta \cdot \appr(\FF_i) + \omega(\FF_i) \right)  + 2\bzeta - \tol(\FF_{i}, \FF_{j})
        \end{split}
        \end{align}
        Plugging in our value for $\tol_{\train}(\FF_{i}, \FF_{j})$ and rearranging, we are able to conclude from \eqref{eq::theorem1proof-eq1} that
        \begin{align*}
            \delta_j \leq 2 \beta  \appr(\FF_i) \leq 2 \beta \xi_i
        \end{align*}
        Therefore, $\balpha_k \leq \balpha_j \leq 2 L \xi_i \leq 2 L \xi_{k^\dagger}$ by the monotone property of both sequences $(\balpha_k)$ and $(\xi_k)$. Finally using $\EE_2$ again, this implies that for all $h \in [H]$
        \begin{align*}
            \| f^k_h - T^*_h f^k_{h + 1} \|_{\mu_h}^2  & \leq L \appr(\FF_{k}) + \omega(\FF_k) \\
            & \leq \beta \cdot  \appr(\FF_{k}) + \delta_k  \\
            & \leq \beta \cdot  \appr(\FF_{k}) + 2 \beta \cdot  \xi_{k^\dagger} \\
            & \leq 3 \beta \cdot  \xi_{k^\dagger} 
        \end{align*}
\end{enumerate}

Observing the bounds from both cases, we are then able to conclude that for whatever $k$ is returned by Algorithm~\ref{alg::theory}, we have the bound
\begin{align*}
    \max_{h \in [H]}  \| f^k_h - T^*_h f^k_{h + 1} \|_{\mu_h}^2 \leq 12  \left( \beta \xi_{k^\dagger} +  \balpha_{k^\dagger} + \bzeta + \omega(\FF_{k^\dagger}) \right) 
\end{align*}
Again, the performance difference lemma ensures that
\begin{align*}
    \regret(\hat \pi) & \leq 2 \sqrt{   12\CC(\mu) H    \left( \beta \xi_{k^\dagger} +  \balpha_{k^\dagger} + \bzeta + \omega(\FF_{k^\dagger}) \right) }.
\end{align*}

We finally conclude by using the fact that $\balpha_k \lesssim \omega(\FF_{k}) + {H^2 \log (M^2 H |\FF_k| /\delta) \over \train}$ for all $k \in [M]$ and that $\train$ and $\valid$ are constant fractions of $n$.

\end{proof}

\section{FQI Algorithm and Guarantees}\label{app::fqi}

\begin{figure}
	\begin{algorithm}[H]
		\caption{ Fitted Q-Iteration}\label{alg::fqi}
		\begin{algorithmic}[1]
			
			\STATE \textbf{Input}: Offline dataset $D = (D_h)$ of $n$ samples for each $h \in [H]$ and model class $\FF$
			\STATE Initialize $f_{H +1} = 0  \in \FF$
			\FOR{$h = H, \ldots, 1$}
			    \STATE 
			    \begin{align*}
			        f_{h} \leftarrow \argmin_{f \in \FF} { 1\over n} \sum_{(x,a, r, x') \in D_h }  \left( f(x, a)  - r - f_{h + 1} (x') \right)^2  
			    \end{align*}
			\ENDFOR
			\RETURN $(f_h)_{h \in [H]}$
			\end{algorithmic}
		
	\end{algorithm}

\end{figure}

Here we state and then prove a more detailed version of the FQI guarantee that was originally stated in Lemma~\ref{lem::fqi}.

\begin{restatable}{lemma}{lemfqidetails}
\label{lem::fqidetails}
	Consider the FQI algorithm (stated in Appendix~\ref{app::fqi} for completeness). For a model class $\FF$, FQI is a $(3, \omega)$-regular base algorithm with 
	$
	\omega_{n, \delta}(\FF) = {200 H^2 \log (16 H|\FF| /\delta ) \over n }
	.$
	
\end{restatable}

\begin{proof}
    This result can be obtained almost immediately from Lemma~\ref{lem::minimization} in the case where the model classes are the same. Observe that $D_h$ is independent of $f_{h + 1}$. Therefore, conditioned on $f_{h + 1}$, we have that
    \begin{align*}
        \| f_h - T^*_h f_{h + 1} \|_{\mu_h}^2  \leq 3 \appr(\FF) + { 40 ( H + 1)^2 \log (2 | \FF  |/\delta) \over n  }
    \end{align*}
    with probability at least $1-\delta$ since $f_h$ is the empirical minimizer. Integrating out the conditioning, taking the union bound over $h \in [H]$, and changing variables to $\delta' = H \delta$  yields the result.
\end{proof}

We may now apply this result to immediately Corollaries~\ref{cor::fqi} and~\ref{cor::fqi-robust}.

\corfqi*

\begin{proof}[Proof of Corollaries~\ref{cor::fqi} and~\ref{cor::fqi-robust}]
We start with Corollary~\ref{cor::fqi-robust}
    Recall that Theorem~\ref{thm::robust} ensures that for an $(\beta, \omega)$-regular algorithm in the case where $k_*$ does not exist, we have
    \begin{align}
	\regret(\hat \pi) \leq C_0 \cdot \min_{k \in [M] } \left\{    \sqrt{ \CC(\mu)  H \left(  \beta \xi_k +   \omega_{\train, \delta/4M} (\FF_k)  + { H^2(\log |\FF_k|  + \iota  ) \over n }\right)   } \right\} 
	\end{align}
	with probability at least $1 - \delta$ for some absolute constant $C_0 > 0$ and $\iota  = \log(M^2 H /\delta)$.
	
	Using Lemma~\ref{lem::fqi}, we may substitute in the values of $\omega$ and $L = 3$ to achieve
	\begin{align*}
	    \regret(\hat \pi)  & \leq C_0 \cdot \min_{k \in [M] } \left\{    \sqrt{ \CC(\mu)  H \left( 3 \xi_k +   {  200H^2 \log(64M |\FF_k| / \delta) \over \train  }   + { H^2(\log |\FF_k|  + \iota  ) \over n }\right)   } \right\} \\
	    & \leq C_0' \cdot \min_{k \in [M] } \left\{    \sqrt{ \CC(\mu)  H \left( \xi_k +   {  H^2 \log(M |\FF_k| / \delta) \over n  }   + { H^2(\log |\FF_k|  + \iota  ) \over n }\right)   } \right\} \\
	    & \leq C_0'' \cdot \min_{k \in [M] } \left\{  \sqrt{ \CC(\mu)  H  \xi_k  } + \sqrt{ \CC(\mu) H^3 \left(\log| \FF_{k} | + \iota \right) \over n  }   \right\}
	\end{align*}
	where $C_0', C_0'' > 0$ are absolute constants. In the second line, we have used the fact $\train$ and $\valid$ are constant fractions of $n$. In the third line, we have used $\sqrt{ a + b} \leq \sqrt{ a} + \sqrt{b}$ for $a, b\geq 0$.
	
	For Corollary~\ref{cor::fqi},
	 when $k_*$ exists, the proof is essentially identical, except that we use Theorem~\ref{thm::main}:
	\begin{align*}
	\regret(\hat \pi) & \leq  C_1 \cdot  \sqrt{  \CC(\mu) H \left( \omega_{\train, \delta /4M} (\FF_{k_*})  + {H^2(\log |\FF_{k_*}|  + \iota  ) \over n} \right) }   \\
	 & \leq C_1' \cdot \sqrt{  \CC(\mu) H \left( { 200 H^2 \log(64 HM | \FF_{k_*} | /\delta ) \over \train }  + {H^2(\log |\FF_{k_*}|  + \iota  ) \over n} \right) }  \\
	 & \leq C_1'' \cdot \sqrt{  \CC(\mu) H^3 \left(\log | \FF_{k_*} | + \iota \right)   \over n  }
	\end{align*}
	where $C_1, C_1', C_1'' > 0$ are all absolute constants.
\end{proof}

	\section{Experiment Details} \label{app::exp}

	\subsection{Practical Implementation of \alg{} for the RL Setting}
	
	\alg{}, as stated in Algorithm~\ref{alg::theory}, is originally designed for the finite horizon case in which there are $H$ functions comprising the value function approximators. For the contextual bandit setting (where $H = 1$), we make no modifications.
	In an effort to further increase the computational and statistical efficiency of \alg{} in the RL setting (as well as to demonstrate that its primary principles are fairly robust), we opted for a discounted infinite horizon implementation with discount factor $\gamma = 0.99$ (default for \textsf{d3rlpy}).
	
	We use a single fixed dataset $D$ (not split into timesteps) and fed this to the Deep Q-Network (DQN) implementation of \cite{seno2021d3rlpy} using all the default hyperparameters except for the network architecture, which was specific to each model class as described in Section~\ref{sec:experiments}. For consistency, we set the number of epochs to 20 across all model classes and experiments for DQN. This generates value function approximators $f^1, \ldots, f^M$. To implement a close approximation of Algorithm~\ref{alg::theory} in discounted case, considered the following procedure. While the algorithm is on model class $k$,  we compute empirical risk minimizers for $k' \geq k$ so that
	\begin{align*}
	    g^{k'} \leftarrow \argmin_{g \in \FF_{k'}} \  {1 \over \train} \sum_{(x, a, r, x') \in D_{\text{train}} } ( g(x, a) - r - \gamma f^k(x'))^2
	\end{align*}
	We then decide whether to switch to $k + 1$ by using the generalization test:
	\begin{align*}
	    \tilde L(g^{k'}, f^k) \geq \tilde L(g^{k}, f^k) - \tol_{\train} ( \FF_k, \FF_{k'}) 
	\end{align*}
    where the functional $\tilde L(\cdot, f^k)$ is the estimated loss on the validation data, as before:
    \begin{align*}
        \tilde L(g, f^k)  = {1 \over \valid} \sum_{(x, a, r, x') \in D_{\text{valid}} } ( g(x, a) - r - \gamma f^k(x'))^2
    \end{align*}
    As noted in Section~\ref{sec:experiments}, we did not find it necessary to tune the any parameters related to $\tol_{\train}(\FF_{k}, \FF_{k'})$ and simply set it to ${d_{k'} \over n}$ where $d_k$ is the dimension of the linar model (for the contextual bandit setting) or the number of hidden nodes in the neural network (for the RL settings), which roughly (up to constants and logarithmic factors) matches known bounds on the pseudo-dimension \citep{bartlett2019nearly}. The lack of necessity to actually make $\tol$ theoretically valid is actually a positive of the algorithm: it shows it is fairly robust in practice and simply matching the order appears to be good enough to generate the current results.  
        To fit the empirical risk minimizers in the CB setting, we simply used ridge regression as in \cite{lee2021model}. To do the same in the RL setting, we trained neural networks of with the same architectures as the DQNs in \textsf{d3rlpy} (state inputs and one output per action to predict the value). We used an Adam optimizer with on 10 epochs with a learning rate of 4e-3 and a batch size of 64. This was implemented through PyTorch~\citep{paszke2019pytorch}.

    One might ask whether it is possible to extend this beyond neural networks with one hidden layer. In practice one can easily use any model, but, in theory, some care may need to be taken in order to set the value of $\tol_{\train}$. For example, to handle more hidden layers, we can appeal to generalized pseudo-dimension bounds~\citep{bartlett2019nearly}. As observed in the current experiments, the setting of $\tol_{\train}$ to rough estimates does not seem to make a huge impact on the results.

    \paragraph{Hold-out baseline} The hold-out method as a model selection baseline was implemented by choosing $k$ that minimizes $\tilde L(f^k, f^k)$ in the RL setting. In the contextual bandit setting it is equivalent to selecting $k$ to minimize $\tilde L(f^k, 0)$, since there is only one step.

	\subsection{Experimental Setups}

	We now describe the specific experimental setup so that it may be reproduced. In order to generate the plots which vary based on the sample size of $D$, we simply curtailed the dataset to the given amount of samples shown on the $x$-axis. Generation of the datasets varied in each domain. It would be interesting in the future to evaluate performance on more stochastic RL environments (the CB evnironment is stochastic) as these are ones we expect toe Hold-out method to do very poorly on. Despite this, our current experiments show it is already sub-optimal even in deterministic settings.

	\paragraph{Contextual Bandit} We replicated almost exactly the study of \cite{lee2021model}. To recap their study, there is a linear contextual bandit with $|\AA| = 10$ and an infinite state space where the linear feature vectors of ambient dimension $d = 200$ for each action are generated by sampling from normal distributions with different covariance matrices. The reward function is generated by taking the inner product of $\theta_*$ with feature vector for action $a \in \AA$. To make this an interesting model selection problem, only the first $d_* = 30$ coordinates are non-zero (although this is not known to the learner) and thus a model class using only the first $d_*$ coordinates is sufficient to solve the problem without any approximation error. The individual model classes were generated by simply truncating the coordinates of the feature vectors to the following sizes $\{ 15, 20, 25, 28, 29, 30, 50, 75, 100, 200\}$. The base algorithm was Algorithm~1 of \cite{lee2021model}.
	
	One difference is that we included several additional model classes to the $d_* = 30$ model that are close enough to \textit{fool} the SLOPE algorithm used in \cite{lee2021model}. This also involved increasing the ambient dimension from $d = 100$ to $d = 200$, but we kept $d_* = 30$. We suspect that this poor performance of SLOPE is due to the fact that SLOPE is heavily dependent on the known deviation bounds whereas \alg{} seems to be comparatively robust. The results of SLOPE seem to be poor whenever the deviation bounds are invalid or too conservative.

	\paragraph{CartPole}
	We used the default dataset from \textsf{d3rlpy} \citep{seno2021d3rlpy} which contains approximately 1500 episodes of a good (but not optimal) behavior policy on the CartPole domain. Everything else remains the same as the standard CartPole environment in Gym \citep{brockman2016openai}. 
	
	\paragraph{MountainCar}
	Since no default dataset for MountainCar is provided in \textsf{d3rlpy}, we generated our own through the following procedure. First, we trained a policy online via SARSA on the discretized environment to achieve good performance on the task. We then collected the offline policy by executing 1000 episodes under the good policy which also took a random action at any time step with probability $0.3$ to induce some coverage on the dataset. To simplify the problem for the base DQN algorithm, we also replaced the sparse reward in the offline dataset with a more dense and informative reward function, giving bonuses for high speeds, proximity to the goal, and achieving the goal. We note that this change is done only to simplify the problem and help the base algorithm solve the task with limited computational resources and tuning so as to increase reproducibility. Everything else remains the same as the standard MountainCar environment in Gym \citep{brockman2016openai}.

	\subsection{Hardware}
	
	Contextual bandit experiments were run on a standard personal laptop with 16 GB of memory and an Intel Core i7 processor.
	RL experiments were run on an internal cluster with 16 GB of memory and an NVIDIA GTX 1080 Ti GPU for PyTorch \citep{paszke2019pytorch}.

\end{document}